\documentclass[a4paper,journal]{IEEEtran}

\IEEEoverridecommandlockouts
\usepackage{graphicx}
\usepackage{epstopdf}
\DeclareGraphicsExtensions{.eps,.pdf} 
\graphicspath{ {figures/} }

\usepackage{cite}
\usepackage{subfigure}
\usepackage{amssymb,amsmath}
\usepackage{algorithmic}
\usepackage{color}
\usepackage{epstopdf}
\usepackage{multirow}
\usepackage{multicol}
\usepackage{cases}
\usepackage{setspace}
\usepackage{amsthm}
\usepackage{caption}
\usepackage{mathrsfs}
\usepackage{delarray}
\usepackage{tikz}
\usepackage[ruled]{algorithm2e}
\usepackage{balance}
\usepackage[mathscr]{euscript}
 \let\mathscr\relax
\usepackage[scr]{rsfso}

\allowdisplaybreaks
\usepackage{bbm}

\newtheorem{theorem}{\bf {Theorem}}

\newtheorem{corollary}{\bf Corollary}
\newtheorem{remark}{{\bf{Remark}}}

\newtheorem{lemma}{\bf {Lemma}}
\newtheorem{assumption}{\bf {Assumption}}

\newcommand{\st}{{\mathrm{s.t.}}}

\newcommand{\Pro}{\mathsf{Prob}}
\newcommand{\UB}{\mathtt{UB}}

\newcommand\argmax{\operatornamewithlimits{argmax}}
\newcommand\argmin{\operatornamewithlimits{argmin}}

\renewcommand{\algorithmicrequire}{\textbf{Input:}}

     \usepackage[compact]{titlesec}
    \titlespacing{\section}{0pt}{2ex}{1ex}
    \titlespacing{\subsection}{0pt}{1ex}{0ex}
    \titlespacing{\subsubsection}{0pt}{0.5ex}{0ex}

\newcommand{\bfa}{\boldsymbol{a}}
\newcommand{\bfr}{\mathbf{r}}
\newcommand{\sfH}{\mathsf{H}}
\newcommand{\sfT}{\mathsf{T}}
\newcommand{\bbE}{\mathbb{E}}

\usepackage{capt-of}
\usepackage{dblfloatfix}
\usepackage{varwidth}
\usepackage{tikz}

\newcommand*\circled[1]{\tikz[baseline=(char.base)]{%
        \node[circle,fill=red!0,draw,inner sep=2pt,opacity=0.5,text opacity=1] (char) {#1};}}

\def\cnt{\stepcounter{enumi}\arabic{enumi}}

\newcommand{\va}{\textbf{\textit{a}}}

\newcommand{\sC}{\mathcal{C}}
\newcommand{\sN}{\mathcal{N}}
\newcommand{\setC}{\mathbb{C}}

\makeatletter

\begin{document}
\bstctlcite{IEEEexample:BSTcontrol}

\title{Network-Aided Intelligent Traffic Steering in 6G O-RAN: A Multi-Layer Optimization Framework}
\author{
	\IEEEauthorblockN{Van-Dinh Nguyen, Thang X. Vu, Nhan Thanh Nguyen, Dinh C. Nguyen,  Markku Juntti,\\ 
\qquad Nguyen Cong Luong, Dinh Thai Hoang, Diep N. Nguyen and  Symeon Chatzinotas}
\thanks{V.-D. Nguyen is with the College of Engineering and Computer Science, VinUniversity, Vinhomes Ocean Park, Hanoi 100000, Vietnam  (e-mail: dinh.nv2@vinuni.edu.vn).}
\thanks{T. X. Vu and S. Chatzinotas are with the Interdisciplinary Centre for Security, Reliability and Trust (SnT), University of Luxembourg, L-1855 Luxembourg City,
Luxembourg (e-mail:            \{thang.vu, symeon.chatzinotas\}@uni.lu).}
\thanks{N. T. Nguyen and M. Juntti are with Centre for Wireless Communications, University of Oulu, P.O.Box 4500, FI-90014, Finland, (email: \{nhan.nguyen, markku.juntti\}@oulu.fi).}
\thanks{N. C. Luong is with the Faculty of Computer Science, PHENIKAA
University, Hanoi 12116, Vietnam (e-mail: luong.nguyencong@phenikaa-uni.edu.vn).}
\thanks{Dinh C. Nguyen is with the Elmore Family School of Electrical and Computer
Engineering, Purdue University, USA (e-mail: nguye772@purdue.edu).}
\thanks{D. T. Hoang and D. N. Nguyen are with the School of Electrical and Data Engineering, University of Technology Sydney, Sydney, NSW 2007, Australia (e-mail:  \{hoang.dinh, diep.nguyen\}@uts.edu.au)}
}

\maketitle
\thispagestyle{empty}
\pagestyle{empty}
\begin{abstract}
To enable an intelligent, programmable and multi-vendor radio access network (RAN) for 6G networks, considerable efforts have been made in standardization and development of open RAN (O-RAN). So far, however, the applicability of O-RAN in controlling and optimizing RAN functions has not been widely investigated. In this paper, we jointly optimize the  flow-split distribution, congestion control and scheduling (JFCS) to enable an intelligent traffic steering application in O-RAN. Combining tools from network utility maximization and stochastic optimization, we introduce a multi-layer optimization framework that provides fast convergence, long-term utility-optimality and  significant delay reduction compared to the state-of-the-art and baseline RAN approaches. Our main contributions are three-fold: $i$) we propose the novel JFCS framework to efficiently and adaptively direct traffic to appropriate radio units; $ii$) we develop low-complexity algorithms based on the reinforcement learning,  inner approximation and bisection search methods to effectively solve the JFCS problem in different time scales; and  $iii)$ the rigorous theoretical performance results are analyzed to show that there exists a scaling factor to improve the tradeoff between delay and utility-optimization. Collectively, the insights in this work will open the door towards fully automated networks with enhanced control and flexibility. Numerical results are provided to demonstrate the effectiveness of the proposed algorithms in terms of the convergence rate, long-term utility-optimality and  delay reduction.
\end{abstract}


\section{Introduction} \label{Introduction}
With the great success of mobile Internet, fifth generation (5G)  cellular networks have been standardized to meet competing demands (\textit{e.g.} extremely high data rate, low-latency and massive connectivity) and proliferation of heterogeneous devices. However, the existing ``one-size-fits-all'' 5G architecture lacks sufficient intelligence and flexibility to  enable the coexistence of these demands \cite{SaadNet2020}. As we move towards 6G, the latest frontier in this endeavor is open radio access network (O-RAN) by disaggregating RAN components and opening up interfaces, which is considered today the most promising approach to revolutionize the wireless technology from ``\textit{connected things}” to ``\textit{connected intelligence}” \cite{LetaiefComMag19,ORAN2020,ORANASAMSUNG2019}. O-RAN is expected to fully enable programmable, intelligent, interoperable and  multi-vendor RAN \cite{ORANAlliance2019}. 

\begin{figure}[t]
	\centering
	\includegraphics[width=0.95\columnwidth,trim={0cm 0.0cm 0cm 0.0cm}]{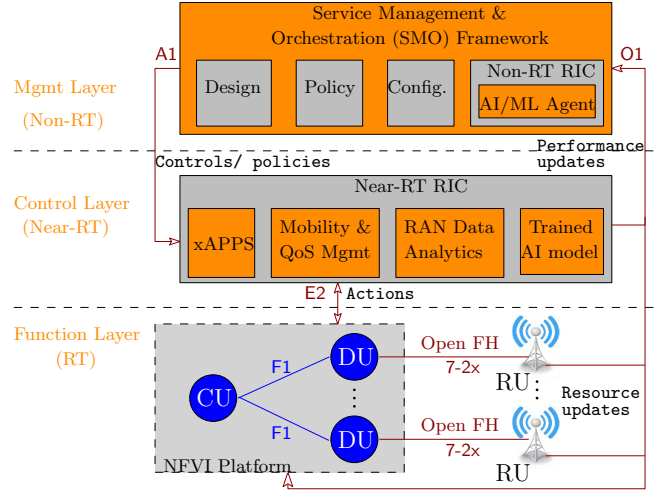}
	\caption{\small O-RAN Alliance reference architecture and workflow \cite{ORAN2020}, with Non-RT and Near-RT RICs. A base station is disaggregated to CU, DU and RU.}
	\label{fig:systemmodel}
\end{figure}

Fig. \ref{fig:systemmodel} illustrates the high-level O-RAN Alliance reference architecture \cite{ORAN2020}, where the “black” parts and interfaces are defined by the 3rd Generation Partnership Project (3GPP), while the “orange” parts and interfaces are defined by O-RAN Alliance. O-RAN initiatives were developed to split the RAN into the radio unit (RU),  distributed unit (DU) and centralized unit (CU), allowing for the interoperability of open hardware (HW), software (SW) and interfaces (\textit{e.g.} O1, A1 and E2) \cite{ORAN2020,ORANASAMSUNG2019}. The O-RAN architecture typically has three main layers (or loops), including the management, control and function layers as illustrated in Fig. \ref{fig:systemmodel}. In particular, the management layer takes place in non-real-time (Non-RT) over 1 s (second) with orchestration, automation functions and trained artificial intelligence (AI) and machine learning (ML) models. The control layer is executed in near real-time (Near-RT) between 10 ms and 1 s to provide functions like radio resource management (RRM), quality-of-service (QoS) management and interference management. Finally, the function layer provides the RAN optimization of a timescale below 10 ms, such as scheduling, power control and radio-frequency assignment, etc. {\color{black}The function layer (CU, DU and RU) is also connected to the Non-RT RIC through the O1 interface for periodic feedback, aiming to fully enable autonomous and self-optimizing networks.} Two important parts introduced in O-RAN are Non-RT RAN intelligent controller (Non-RT RIC) and Near-RT RIC that allow to access RRM functions. The former enables  AI/ML workflow for RAN components and RRM like traffic steering (TS) as well as policy-based guidance of applications  in Near-RT RIC, while the latter is embedded with the control/optimization algorithms of RAN and radio resources \cite{BonatiComMag21,GavrilovskaWPC2020,ORANAlliance2019,WangOpenRAN2019}.

\subsection{Motivation}
{\color{black}Yet the existing research efforts on O-RAN in the academic community are isolated, providing only tailored solutions to problems at either the physical or higher layers \cite{KumarGC2020,Pamukluicc2021,LeeGC2020,RomeroINFOCOM2021}. On the other hand, ML-based approaches (e.g. \cite{KumarGC2020,YangTWC22,Pamukluicc2021,RomeroINFOCOM2021,MotallebTNSM23,LienTII23}) often ignored periodic feedback loops and assumed that the RAN information is available at SMO to perform resource allocation and RAN management, making a fully automated network impractical.
The understanding of how O-RAN could help improve network performance by controlling data traffic and optimizing RAN functions remains rather limited in the literature. In this paper, we aim to fill this gap by conducting an in-depth analysis of the multi-layer design between the physical and higher layers and developing low-complexity algorithms for network control, scheduling and resource allocation in different time scales. We also analyze their impact on the throughput and delay performances in the 6G O-RAN context.}

In light of the above discussions, this paper focuses on designing the TS control to intelligently direct the user traffic through a group of RUs, taking into account available resources and users' service requirements.  To fully realize the potential performance of the TS scheme, O-RAN allows customization of user-centric strategies, multi-path routing and multi-connectivity as well as proactive optimization of network parameters through RICs. However, the problem becomes more challenging in the O-RAN setting due to several complicating factors: $i)$ the traffic demand of user equipements (UEs) often varies over time, and the complete information of the RAN layer is indeterminate at the time of optimization algorithm execution. Hence, the policies and control decisions at the service management and orchestration (SMO) must be adapted to the variation of data traffic; $ii)$ the total data traffic is distributed unevenly to RUs due to different downlink (DL) throughput capabilities, causing high queueing delay; and $iii)$ the strong correlation between  congestion control and scheduling optimization influences the optimal choice of flow-split distribution of data traffic across all RUs. In addition, the deployment of fully automated  networks is an intricate problem in O-RAN that calls for intelligent, scalable
and self-organizing strategies for a holistic multi-layer optimization framework. In this regard, reinforcement learning (RL) plays an important role in achieving long-term utility optimization. \textit{To the best of our knowledge, the TS optimization problem for O-RAN as outlined above has not been thoroughly addressed in the literature.}

\subsection{Main Contributions}
In this paper, we consider a practical scenario where the complete information of the RAN layer is not available at the beginning of each time-frame. Instead, we assume that only their expected values are available to approximately measure queueing delay.  An interesting question naturally arises: \textit{How does the incomplete information of  user traffic demands affect the optimal choices of the TS scheme?} To answer this  question and address the challenges above, we introduce a holistic multi-layer optimization framework that jointly optimizes the flow-split distribution, congestion control and scheduling (called JFCS). The proposed framework effectively characterizes the complex interactions between layers (\textit{e.g.} flow-split selection, congestion control rate and power allocation). In summary, we make the following three key contributions:
\begin{itemize}
    \item We propose a novel JFCS framework to efficiently and adaptively direct traffic to appropriate RUs. Our framework not only generalizes the classical queue-length-based congestion control and scheduling (QCS) method \cite{NeelyToN2008}, but also provides a synergy between RL, QCS and updated network state information, and thus enabling a closed-loop control of the TS in the O-RAN context.
   
    \item {\color{black} To ensure the practicality, we identify inherent properties of the JFCS problem and propose an intelligent resource management algorithm to solve it effectively by leveraging the stochastic optimization framework \cite{neely2010stochastic}. In particular, by exploiting the historical system information  accumulated from the previous time-slots, an RL process is developed to build the smoothed best response while maximizing the long-term utility for each data-flow under arbitrary changes in traffic demands. Given the updated queue-length vector and the optimal flow-split distribution,  two low-complexity algorithms are developed to effectively solve the short-term power control optimization subproblem in an iterative fashion.}
    
    \item Given a scaling factor $\varphi$ to minimize the Lyapunov drift \cite{EryilmazJSAC2006}, the theoretical performance results are analyzed to show that the queueing network is stable. In addition, the expected divergence in queue-length and the optimality gap of congestion control rate still scale as $\mathcal{O}(\sqrt{\varphi})$ and $\mathcal{O}(1/\sqrt{\varphi})$, respectively. Thus, there always exists a scaling factor to balance utility-optimality and latency.    
\end{itemize}
We numerically evaluate the performance of the proposed framework. Results show that the proposed framework can improve network resource utilization significantly while achieving fast convergence and long-term utility-optimality,  compared to state-of-the-art approaches.

\subsection{Paper Organization and Mathematical Notation}
The remainder of this paper is organized as follows. The related work is discussed in Section \ref{sec_RW}. In Section \ref{PreliminariesDefinitions}, we first introduce the network model and then present the problem formulation. The proposed JFCS framework and its solutions are provided in Sections \ref{sec:ORANNUM} and \ref{sec_RAAlgo}, respectively. Section \ref{sec_PerfAnaly} presents the key theoretical performance results of the JFCS framework. 
 Numerical results are given in Section \ref{sec_NumericalResults}, while Section \ref{sec_Conclusion} concludes the paper.

\textit{Mathematical notation:} Throughout this paper, matrices and vectors are written as bold uppercase and lowercase letters, respectively, while the scalar number is denoted in lowercase. $\mathbf{h}^{\sfH}$ is the Hermitian transpose of vector $\mathbf{h}$. The notation $x \sim \sC\sN(0,\sigma^2)$ implies that $x$ is a circularly-symmetric complex Gaussian random variable with zero mean and variance  $\sigma^2$. $\|\cdot\|$ stands for the vector's Euclidean norm.  $\setC$ and $\mathbb{R}$ denote the sets of all complex and real numbers, respectively. Finally, $\mathbb{E}\{\cdot\}$ denotes the expectation of a random variable.

\section{Related Work}\label{sec_RW}
Multi-layer (a.k.a. cross-layer) optimization for traditional cellular  RAN architectures has been extensively studied in the literature (see \textit{e.g.}, \cite{HabibiACCESS19} and references therein). For example, Tang \textit{et al.} \cite{TangTWC15}  studied a multi-layer resource allocation problem to minimize the overall  system power consumption in a cloud-RAN (C-RAN), which jointly optimizes the service scaling,  remote radio head selection, and beamforming. In \cite{LuongTWC18}, a joint design of virtual computing and radio
resource allocation was proposed. It was shown that this approach can efficiently allocate the virtual computing of the baseband unit (BBU) pool to achieve load balancing among users with significantly reduced power consumption. These problems are often solved by the difference of the convex algorithm due to the combinatorial nature and strong coupling between optimization variables. To address this challenge,  graph theory techniques were introduced in \cite{DouikTWC16} and \cite{DouikCL17} to effectively solve the jointly coordinated scheduling and power optimization problem in C-RAN. Recently, the multi-layer network coding was also investigated in \cite{AbiadTMC19,DouikTWC17,AbiadTMC22}, taking into account the rate heterogeneity of different users to remote radio heads. In general, these existing works only optimized   radio resources, while other factors at higher layers (\textit{e.g.} congestion control and routing) were overlooked, making guaranteed multi-layer QoS for O-RAN  infeasible. In addition, the non-causal statistical knowledge of traffic demands is required to model queue states, which is again impractical.

So far, there have been only a few attempts to study the  applicability  of the O-RAN architecture. Kumar \textit{et al.} \cite{KumarGC2020} proposed an automatic 
relation (ANR) approach to manage neighbour cell relationships by leveraging ML techniques, hence improving gNodeB (gNB) handovers. The work in \cite{YangTWC22} introduced an intelligent user access control algorithm based on deep reinforcement learning, aiming to maximize the overall throughput and avoid frequent handovers.
The authors in \cite{Pamukluicc2021} developed an RL-based dynamic function splitting which is shown to be able to effectively decide the O-RAN's function splits and reduce operating costs. Based on the Working Group (WG)-2 AI/ML specifications of the O-RAN Alliance,  Acumos framework and open network automation platform were introduced in  \cite{LeeGC2020} to generate AI/ML models to be deployed in RIC modules and  monitor the designed workflow, respectively. {\color{black} Motalleb \textit{et al.} \cite{MotallebTNSM23} developed an iterative algorithm to jointly optimize service-aware baseband resource allocation and virtual network function activation, thus achieving better data rate and lower end-to-end delay. Very recently, a deep reinforcement learning-based intelligent session management for ultra-reliable and low latency communications (URLLC) was proposed in \cite{LienTII23}  to allocate resources for serving current  and new sessions more efficiently. However,  these studies did not reveal any observable information about the RAN layer to SMO via periodic feedback loops. Thus,  RICs in these studies were unable to monitor RAN in a timely manner to enable their management automation within O-RAN.
}

{\color{black}In traditional RAN architectures, the TS solutions are typically determined by users' radio conditions of a serving cell while treating signals from neighboring cells as interference \cite{ORANWG22021}. The authors in \cite{AnwarIoT22} proposed a distributed TS scheme through edge servers, where the matrix-based  shortest path selection and matrix-based  multipath searching algorithms were developed to dynamically determine the optimal paths for traffic steering. Very recently, Kavehmadavani \textit{et al.} \cite{KavehmadavaniTWC23} showed that a dynamic multi-connectivity (MC)-based TS scheme can help steer traffic flows towards the most suitable cells based on user-centric conditions. In addition, the flow split for each user was purely determined by the RUs' capacity in delivering user traffic demands, resulting in a very suboptimal solution. However, this work did not embed AI/ML solutions in Non-RT RIC within the O-RAN architecture  and assumed that all network information is available at Near-RT RIC to optimize radio resource allocation.}

Different from all the above works and others in the literature that focus on a single layer, we propose a fully multi-layer optimization framework that captures interplays between the physical and higher layers, enabling proactive optimization of network parameters through RICs with periodic feedback loops. This holistic multi-layer optimization framework guarantees the long-term utility-optimality with far less latency than state-of-the-art approaches, opening the door towards fully automated networks with enhanced control and flexibility.


\section{Network Model and Problem Formulation} \label{PreliminariesDefinitions}

\subsection{Network Model}
\begin{figure}[h]
	\centering
	\includegraphics[width=0.9\columnwidth,trim={0cm 0.0cm 0cm 0.0cm}]{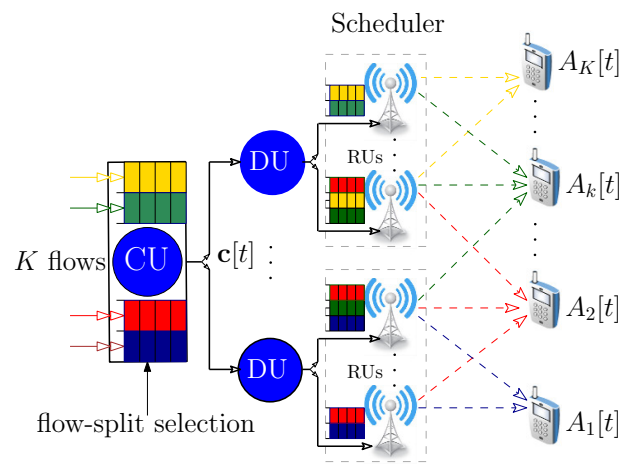}
	\caption{\small Illustration of the O-RAN-based system model enabling TS where each DU connects to multiple RUs towards cost-effective deployment. }
	\label{fig:RUUEAssociation}
\end{figure}

\begin{figure}[t]
	\centering
	\includegraphics[width=1\columnwidth,trim={0cm 0.0cm 0cm 0.0cm}]{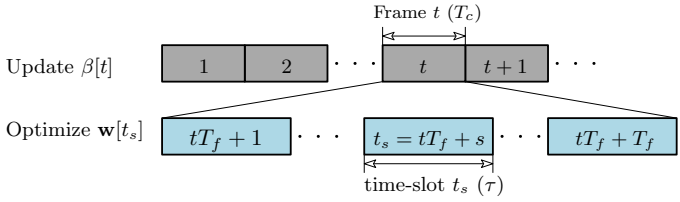}
	\caption{\small Illustration of frame structure with each time-frame $t$ (corresponding to one large-scale coherence time) consisting of $T_f$ time-slots.}
	\label{fig:Timeframe}
\end{figure}

As shown in Fig. \ref{fig:RUUEAssociation}, we consider an O-RAN architecture with one CU, $I$ DUs and $J$ RUs, where each DU connects to multiple RUs for cost-effective deployment. Let us denote by $\mathcal{I}\triangleq\{1,2,\cdots, I\}$ the set of DUs.  We consider a downlink multi-user multiple-input single-output (MU-MISO) system, where $J$ RUs simultaneously serve the set $\mathcal{K}\triangleq\{1,2,\cdots,K\}$ of $K=|\mathcal{K}|$ single-antenna UEs.  The $j$-th RU served by  the $i$-th DU  is referred to as RU $(i,j)$, which is equipped with $M_{i,j}$ antennas. The total number of RUs' antennas is thus $M_{\Sigma}=\sum_{\forall(i,j)}M_{i,j}$. The set of RUs served by DU $i$ is denoted by $\mathcal{J}_i\triangleq\{(i,1),\cdots,(i, J_i)\}$  with $|\mathcal{J}_i| = J_i$ and $\sum_{i\in\mathcal{I}}J_i=J$. The total set of RUs is denoted as $\mathcal{J}\triangleq \cup_{i\in\mathcal{I}}\mathcal{J}_i$.
 We assume that the midhaul (MH) link between the CU and DU and fronthaul link between the DU and RU have sufficient capacity (\textit{i.e.}, high-speed optical ones), so that the transmission latency from CU to RUs and queueing latency at CU and DUs are negligible.

We consider that the system operates in a discrete time-frame indexed by $t\in[1,2,\cdots, T]$, which corresponds to one large-scale coherence time with a duration of  $T_c$, as illustrated in Fig. \ref{fig:Timeframe}. Each frame is divided into $T_f$ time-slots of equal duration $\tau=T_c/T_f$, where the time-slot is indexed by $t_s = tT_f + s$ with $s\in\{1,2,\cdots,T_f\}$. 
At CU, there exist $K$ independent data-flows, each of which is intended for one UE.  
The CU splits the data-flow of UE $k$, say flow $k$, into multiple sub-flows which are possibly transmitted through the set of paths and then aggregated at this UE \cite{VuTWC2019,SinghCOMML2016}, so-called ``traffic steering''. For  data-flow $k$, we denote by $\mathcal{P}_k\triangleq \{(i,j)\}_{\forall (i,j)\in\mathcal{J}}$ the set of path states, including queue states and routing tables.
To improve the system throughput, a subset of separate paths in the set $\mathcal{P}_k$ (\textit{i.e.}, via neighboring RUs indexed by ($i,j$)) should be appropriately selected. Let us denote by $\mathbf{c}_k[t]\triangleq\bigl[c_k^{i,j}[t]\bigr]_{(i,j)\in\mathcal{P}_k}$ the flow-split selection (action) vector for data-flow $k$ in time-frame $t$, \textit{i.e.}, $c_k^{i,j}[t]=1$ if path ($i,j$) $ \in\mathcal{P}_k$ (\textit{i.e.}, via RU $(i,j)$) is selected to transmit data of flow $k$; otherwise, $c_k^{i,j}[t]=0$. We let $\beta_k^{i,j}[t]\in[0,1]$ be the fraction of data-flow $k$  which is routed via path $(i,j)$ in time-frame (state) $t$ by selecting action $c_k^{i,j}[t]$, where $\sum_{(i,j)\in\mathcal{P}_k }\beta_k^{i,j}[t]=1$. The global flow-split decision is denoted by $\mathscr{B}[t]\triangleq\{\boldsymbol{\beta}_k[t], \forall k\bigl| \sum_{(i,j)\in\mathcal{P}_k }\beta_k^{i,j}[t]=1, \forall k\}$, where each column flow-split vector $\boldsymbol{\beta}_k[t]\triangleq\bigr[\beta_k^{i,j}[t]\bigl]_{(i,j)\in\mathcal{P}_k }^{\sfT}\in\mathbb{R}^{J}$ corresponds to the flow-split vector of data-flow $k$.

\subsubsection{Wireless Channel Model and Downlink Throughput}
The large-scale fading coefficients  are assumed to be invariant within one frame $T_c$, while the small-scale fading components with
a low degree of mobility are assumed to be unchanged during time-slot $t_s$ with  duration of $\tau$ and vary independently in the next time-slot. For example, the large-scale fading coefficients may stay invariant for a period of at least 40 small-scale fading coherence intervals for indoor scenarios \cite{Ngo:TWC:Mar2017}.
The channel vector between RU $(i,j)$  and UE $k\in\mathcal{K}$ in time-slot $t_s$ is denoted by $\mathbf{h}^{i,j}_{k}[t_s]\in\mathbb{C}^{M_{i,j}\times 1}$, which follows the Rician fading model  with the Rician factor $\kappa^{i,j}_{k}[t]$. In particular, $\mathbf{h}_{k}^{i,j}[t_s]$ is modeled as $\mathbf{h}_{k}^{i,j}[t_s]  = \sqrt{\xi_{k}^{i,j}}[t]\Bigl(\sqrt{\kappa^{i,j}_{k}[t]/(\kappa^{i,j}_{k}[t]+1)}\bar{\mathbf{h}}_{k}^{i,j}[t] + \sqrt{1/(\kappa^{i,j}_{k}[t]+1)}\tilde{\mathbf{h}}_{k}^{i,j}[t_s] \Bigr)$
where $\xi_{k}^{i,j}[t]$ represents the large-scale fading; $\bar{\mathbf{h}}_{k}^{i,j}[t]$ and $\tilde{\mathbf{h}}_{k}^{i,j}[t_s]\sim \mathcal{CN}(0,\mathbf{I})$ are the line-of-sight (LoS) and non-LoS (NLoS) components, which follow   a  deterministic  channel  and   Rayleigh  fading  models,  respectively.  We let $\mathbf{H}[t_s]\triangleq \bigr[\mathbf{h}_{1}[t_s]\cdots\mathbf{h}_{K}[t_s]\bigr]\in\mathbb{C}^{M\times K}$ denote the channel matrix between all RUs and UEs in time-slot $t_s$ where $\mathbf{h}_{k}[t_s]\triangleq\bigl[(\mathbf{h}_{k}^{i,j}[t_s])^{\sfH}\bigr]^{\sfH}_{\forall i,j}\in\mathbb{C}^{M\times 1}$ corresponds to the channel vector between RUs and UE $k$.

Let us denote by $x_{k}^{i,j}[t_s]$ and $\mathbf{w}_{k}^{i,j}[t_s]\in\mathbb{C}^{M_{i,j}\times 1}$ a unit-power data symbol and a linear beamforming vector transmitted from RU $(i,j)$ to UE $k$ in time-slot $t_s$, respectively. The received signal at UE $k$ in time-slot $t_s$ can be written as
\begin{align}
   & y_k[t_s] = \sum_{(i,j)\in\mathcal{P}_k }(\mathbf{h}_{k}^{i,j}[t_s])^{\sfH}\mathbf{w}_{k}^{i,j}[t_s]x_{k}^{i,j}[t_s] \nonumber\\ 
    &  + \sum_{k'\in\mathcal{K}\setminus\{k\}}\sum_{(i,j)\in\mathcal{P}_{k'} }(\mathbf{h}_{k}^{i,j}[t_s])^{\sfH}\mathbf{w}_{{k'}}^{i,j}[t_s]d_{k'}^{i,j}[t_s] + \omega_k[t_s]
\end{align}
where $\omega_k[t_s]$ is the additive white Gaussian background noise (AWGN) with power $N_0$. The downlink achievable rate (bits/s) of UE $k$ from RU $(i,j)$ in time-slot $t_s$ can be written as $r_{k}^{i,j}(\mathbf{w}[t_s])\triangleq W\log_2\bigl(1 + \gamma_{k}^{i,j}(\mathbf{w}[t_s])\bigr)$, where $W$ is the system bandwidth and the  signal-to-interference-plus-noise ratio (SINR) $\gamma_{k}^{i,j}(\mathbf{w}[t_s])$ is given by $\gamma_{k}^{i,j}(\mathbf{w}[t_s])=|(\mathbf{h}_{k}^{i,j}[t_s])^{\sfH}\mathbf{w}_{k}^{i,j}[t_s]|^2/\Phi_k^{i,j}(\mathbf{w}[t_s])$ with
\begin{align}
   &\Phi_{k}^{i,j}(\mathbf{w}[t_s]) \triangleq \underbrace{\sum_{(i',j')\in\mathcal{P}_k\setminus\{(i,j)\}}|(\mathbf{h}_{k}^{i',j'}[t_s])^{\sfH}\mathbf{w}_{k}^{i',j'}[t_s]|^2}_{\text{Intra-user interference}} \nonumber\\
    &\qquad\quad + \underbrace{ \sum_{k'\in\mathcal{K}\setminus\{k\}}\sum_{(i,j)\in\mathcal{P}_{k'}}|(\mathbf{h}_{k}^{i,j}[t_s])^{\sfH}\mathbf{w}_{k'}^{i,j}[t_s]|^2}_{\text{Inter-user interference}} +N_0
\end{align}
\noindent and  $\mathbf{w}[t_s]\triangleq\bigl[(\mathbf{w}_{k}^{i,j}[t_s])^{\sfH}\bigr]^{\sfH}_{k\in\mathcal{K},(i,j)\in\mathcal{P}_k}$ being the vector embracing all the beamformers.
The overall effective  data rate of data-flow $k$ (or UE $k$) can be computed as
$r_k(\mathbf{w}[t_s]) = \sum_{(i,j)\in\mathcal{P}_k}$ $ r_{k}^{i,j}(\mathbf{w}[t_s])$. Then, for  each $\mathbf{H}[t_s]$ and a given $\boldsymbol{\beta}_k[t]$, we define the \textit{instantaneous  achievable
rate region} under  beamformer $\mathbf{w}[t_s]$ as
\begin{IEEEeqnarray}{cl}
\mathscr{C}_{\mathbf{H}[t_s]}\triangleq \Bigg\{r_k(\mathbf{w}[t_s]), \forall k\Biggl|  {\begin{matrix}r_k(\mathbf{w}[t_s]) = \sum\limits_{(i,j)\in\mathcal{P}_k}r_{k}^{i,j }(\mathbf{w}[t_s]) \\
\sum\limits_{k\in\mathcal{K}}\|\mathbf{w}_{k}^{i,j}[t_s]\|_2^2 \leq P_{\max}^{i,j}, \forall (i,j) \end{matrix}} \Bigg\}\nonumber\\ 
\end{IEEEeqnarray}
where $P_{\max}^{i,j}$ denotes the transmit power budget of RU $(i,j)$. We note that the achievable rate $r_{k}^{i,j}(\mathbf{w}[t_s])$ is upper bounded by $r_{k}^{i,j}(\mathbf{w}[t_s]) \leq W\log_2\bigl(1 + P_{\max}^{i,j}\bigl\|\mathbf{h}_{k}^{i,j}[t_s]\bigr\|_2^2/N_0\bigr)$ for a limited transmit power budget $P_{\max}^{i,j}$, leading to $r_k(\mathbf{w}[t_s]) < \infty, \forall k,t$.

\subsubsection{Queueing Model} 
As illustrated in Fig. \ref{fig:RUUEAssociation}, each RU maintains a separate queue for each UE. Let $A_k[t]$ (bits/s) be the total rate of data arriving at RU destined for UE  $k$ in time-frame $t$ with mean $\mathbb{E}\{A_{k}\}=\bar{A}_k$. We assume that $A_k[t]$ is random and upper bounded by a finite constant $A^{\max}$, such as $A_k[t] \leq A^{\max}<\infty, \forall k,t,$ and  unknown at  the beginning
of time-frame $t$. As a result, the queue-length of data-flow $k$ at RU $(i,j)$ in time-slot $t_s$ evolves as follows: $q_k^{i,j}[t_{s+1}] = \Bigl[q_k^{i,j}[t_s] + \beta_k^{i,j}[t]A_k[t]\tau - r_{k}^{i,j}(\mathbf{w}[t_s])\tau \Bigr]^+$, where $[x]^+\triangleq\max\{0,x\}$. By $\mathbf{q}[t_s]\triangleq\bigl[q_k^{i,j}[t_s] \bigr]^{\sfT}_{k,(i,j)}$ and following \cite{EryilmazJSAC2006}, a queueing network is \textit{stable}  if the steady-state total queue-length remains finite, such as
\begin{align}\label{steady-state}
    \underset{t_s\to\infty}{\limsup}\ \mathbb{E}\{\|\mathbf{q}[t_s]\|_1\} < \infty.
\end{align}

\subsection{Problem Formulation}
Let $\bar{r}_k\triangleq \underset{t_s\to\infty}{\lim}\frac{1}{t_s}\sum_{\ell=1}^{t_s} r_k(\mathbf{w}[\ell])$ denote the long-term average rate of data-flow $k$. Each UE $k$ is associated with a utility function, denoted by $U_k(\bar{r}_k)$. 
To facilitate the analysis presented later, we make the following assumption to the utility function \cite{VuTWC2019,EryilmazJSAC2006,KeyInforcom2007,LiuJSAC2017}. 
\begin{assumption}\label{assp:1}
The utility function $U_k(\cdot)$ is assumed to satisfy the following conditions
\begin{itemize}
    \item 
    $U_k(\cdot)$ is twice continuously differentiable,  increasing, and strictly concave.
    \item There exist positive constants $0 < \psi < \Psi < \infty$, such as $\psi\leq - U_{k}^{''}(\bar{r}_k) \leq \Psi, \forall \bar{r}_k\in[0,\ \bar{r}^{\max}]$, with $\bar{r}^{\max}$ being the maximum  long-term  average  rate  of any data  flow.
\end{itemize}

\end{assumption}

Our goal is to maximize the network utility function $\sum_{k\in\mathcal{K}}U_k(\bar{r}_k)$, subject to the probabilistic delay constraint,  achievable rate region and  queue-stability constraint. Based on the network utility maximization (NUM) framework, the joint flow-split distribution, congestion control and scheduling optimization problem (JFCS) can be mathematically formulated as
\begin{subequations} \label{eq:JFCS1}
	\begin{IEEEeqnarray}{cl}
		\textbf{JFCS}:\ &\underset{\boldsymbol{\beta}, \bar{\mathbf{r}}, \mathbf{w}}{\mathrm{max}} \   \sum_{k\in\mathcal{K}}U_k(\bar{r}_k)\label{eq:JFCS1a} \\
		& \st  \  \underset{t_s\to\infty}{\limsup}\ \mathbb{E}\{\|\mathbf{q}[t_s]\|_1\} < \infty                       \label{eq:JFCS1b}\\
		     & \qquad  r_k(\mathbf{w}[t_s])\in \mathscr{C}_{\mathbf{H}[t_s]}, \forall t_s, k\in\mathcal{K}                        \label{eq:JFCS1c}\\
		    &\qquad \boldsymbol{\beta}_k[t]\in \mathscr{B}[t], \forall t, k\in\mathcal{K}\label{eq:JFCS1d} \\
		     &\qquad \Pro\Bigl(\frac{q_k^{i,j}[t_s]}{\bar{A}_k} \leq \bar{d}_k\Bigr) \geq \epsilon_k,\ \forall t_s, k, (i,j) \label{eq:JFCS1e} \quad
	\end{IEEEeqnarray}
\end{subequations}
where $\boldsymbol{\beta}\triangleq\bigl[\boldsymbol{\beta}_k^{\sfT}\bigr]^{\sfT}_{k\in\mathcal{K}}$ and $\bar{\mathbf{r}}\triangleq\bigr[\bar{r}_k\bigr]^{\sfT}_{k\in\mathcal{K}}$. {\color{black}Constraint \eqref{eq:JFCS1e} ensures different minimum outage delay requirements for sub-flows, where $\bar{d}_k$ and $\epsilon_k$ ($0\ll\epsilon_k \leq 1$) are the maximum allowable average delay and the required reliable communication for each UE, respectively. It is stated that the probability of $\frac{q_k^{i,j}[t_s]}{\bar{A}_k} \leq \bar{d}_k$ (\textit{i.e.} UEs’ maximum allowable delay) should be greater than or equal to a certain positive constant $\epsilon_k$. This probabilistic  constraint is used to tackle the randomness and variability of the arrival rate while  maintaining a certain level of performance.}

\begin{remark}
It is clear that  problem \eqref{eq:JFCS1} needs to be executed in different time scales (\textit{i.e.}, over the long-term scale $t$ at Non-RT RIC and the short-term scale $t_s$ at Near-RT RIC), as shown in Fig. \ref{fig:Timeframe}. In particular, the global  flow-split vector  $\boldsymbol{\beta}[t]$ is only updated once per time-frame $t$ to reduce computational complexity and information exchange, as well as to ensure a stable queueing system. On the other hand, the beamforming vector $\mathbf{w}[t_s]$ and the instantaneous  achievable  rate $\mathbf{r}[t_s]$ are optimized  based on the real-time effective CSI $\mathbf{H}[t_s]$ in time-slot $t_s$,  adapting to dynamic environments. 
\end{remark}

\section{JFCS-based Network Utility Optimization}\label{sec:ORANNUM}

\subsection{Tractable Form of the JFCS Problem \eqref{eq:JFCS1}}\label{sec:ORANNUM_A}
\textbf{Challenges of Solving JFCS Problem \eqref{eq:JFCS1}}:
We can observe that constraint \eqref{eq:JFCS1c} is nonconvex while  \eqref{eq:JFCS1e} is a nonconvex probabilistic constraint, generally making  problem \eqref{eq:JFCS1} NP-hard. In addition, the expectations in the constraints cause the stochastic nature of the problem, which cannot be solved directly. The classical optimization approaches, such as successive convex approximation (SCA) \cite{razaviyayn2014successive}, are often applied to solve the optimization problems of nonconvex and deterministic constraints. However, the stochastic SCA-based algorithms can no longer guarantee a feasible and (sub)-optimal solution of all  subsequent time intervals (TTIs)  due to  the dynamics of the physical layer at  small timescales. The  flow-split decisions  mainly rely on the previous states updated by the RAN layer.
Towards practical applications, an efficient and adaptive solution to the long-term subproblem of \eqref{eq:JFCS1} is necessary to achieve high QoE for all UEs in every TTI.  

Let us start by transforming problem \eqref{eq:JFCS1} into a more tractable form. Towards a safe design, we consider the replacement of constraint \eqref{eq:JFCS1e} by its deterministic constraint. From the basic property of probability, we can rewrite \eqref{eq:JFCS1e} as $\Pro\bigl(q_k^{i,j}[t_s] \geq \bar{A}_k\bar{d}_k\bigr) \leq 1-\epsilon_k$. It follows from the well-known Markov inequality \cite{BillingsleyProbability} that $\Pro\bigl(q_k^{i,j}[t_s] \geq \bar{A}_k\bar{d}_k\bigr) \leq \mathbb{E}\{q_k^{i,j}[t_s]\}/\bar{A}_k\bar{d}_k$, yielding
\begin{align}\label{eq:JFCS1e_Relaxed}
  &\sum\nolimits_{\ell=1}^{t}\beta_k^{i,j}[\ell]\bar{A}_k\tau - (1-\epsilon_k)\bar{A}_k\bar{d}_k - \sum\nolimits_{\ell=1}^{t_{s-1}}r_{k}^{i,j}(\mathbf{w}[\ell])\tau\nonumber\\
  &  \leq r_{k}^{i,j}(\mathbf{w}[t_s])\tau,\ \forall t_s, k\in\mathcal{K}, (i,j)\in\mathcal{P}_k
\end{align}
where each queue-length is always non-negative. We note that  \eqref{eq:JFCS1e_Relaxed} is a relaxed constraint of \eqref{eq:JFCS1e}, which implies that any feasible of the former is also feasible for the latter but not vice versa due to the Markov upper bound on the outage probabilities.

To facilitate the following optimization, we introduce congestion control variables $\boldsymbol{a}[t_s]\triangleq\bigl[a_k[t_s]\bigr]_{k\in\mathcal{K}}^{\sfT}$, satisfying $\bar{a}_k-\bar{r}_k\leq 0, \forall k$, where $\bar{a}_k\triangleq \underset{t_s\to\infty}{\lim}\frac{1}{t_s}\sum_{\ell=1}^{t_s} a_k[\ell]$. Problem \eqref{eq:JFCS1} is then rewritten as
\begin{subequations} \label{eq:JFCS2}
	\begin{IEEEeqnarray}{cl}
		\ &\underset{\boldsymbol{\beta}, \bar{\boldsymbol{a}}, \bar{\mathbf{r}}, \mathbf{w}}{\mathrm{max}} \   \sum_{k\in\mathcal{K}}U_k(\bar{a}_k)\label{eq:JFCS2a} \\
		& \st  \quad\  \eqref{eq:JFCS1b}, \eqref{eq:JFCS1c},     \eqref{eq:JFCS1d},   \eqref{eq:JFCS1e_Relaxed}               \label{eq:JFCS2b}\\
		     & \qquad\quad  \bar{a}_k-\bar{r}_k\leq 0, \forall k.  \label{eq:JFCS2c} 
	\end{IEEEeqnarray}
\end{subequations}
We also introduce a new auxiliary queue-length vector $\hat{\mathbf{q}}[t_s]\triangleq\bigl[\hat{q}_k[t_s] \bigr]^{\sfT}_{k\in\mathcal{K}}$, where
$\hat{q}_k[t_{s+1}] = \bigl[\hat{q}_k[t_s] + a_k[t_s]\tau - r_k(\mathbf{w}[t_s])\tau \bigr]^+$ to associate constraint \eqref{eq:JFCS2c} with a penalty function and $a_k[t_s]\in[0,A^{\max}]$. We define the total queue backlog of all UEs in time-slot $t_s$ as  $L[t_s] = \frac{1}{2}\bigl(\sum_{k\in\mathcal{K}}$ $\sum_{(i,j)\in\mathcal{P}_k} \frac{q_k^{i,j}[t_s]^2}{\tau^2} + \sum_{k\in\mathcal{K}}\frac{\hat{q}_k[t_s]^2}{\tau^2}\bigr)$, which is the quadratic Lyapunov function \cite{neely2010stochastic,TassiulasTAC92}. For given $(\mathbf{q}[t_s],\hat{\mathbf{q}}[t_s])$, the Lyapunov drift from  time-slot $t_s$ to $t_{s+1}$ is given as $\Delta L[t_s] = L[t_{s+1}] - L[t_s]$.  To guarantee joint network stability and penalty minimization  (\textit{i.e.}, \eqref{eq:JFCS1b} and \eqref{eq:JFCS2c} hold true), we adopt the drift-plus-penalty procedure \cite{neely2010stochastic} to minimize  the drift of a quadratic Lyapunov function and rewrite \eqref{eq:JFCS2} as
\begin{subequations} \label{eq:JFCS3}
	\begin{IEEEeqnarray}{cl}
		\ &\underset{\boldsymbol{\beta}, \bar{\boldsymbol{a}}, \bar{\mathbf{r}}, \mathbf{w}}{\mathrm{max}} \quad   \varphi\sum_{k\in\mathcal{K}}\mathbb{E}\{U_k(a_k[t_s])\} 
		- \mathbb{E}\{\Delta L[t_s]\}
		\label{eq:JFCS3a} \\
		& \st  \quad\   \eqref{eq:JFCS1c},     \eqref{eq:JFCS1d},   \eqref{eq:JFCS1e_Relaxed}               \label{eq:JFCS3b}
	\end{IEEEeqnarray}
\end{subequations}
where $\varphi$ is a scaling factor to balance two objective functions. 
 We now show that constraint \eqref{eq:JFCS2c} holds with equality at optimum by introducing the following lemma.
\begin{lemma}\label{lem_1}
 For each data-flow of UE $k$, the optimal congestion control rate  is  equal to the optimal long-term average service rate, \textit{i.e.}, $\bar{a}_k^{*}-\bar{r}_k^{*}= 0, \forall k.$
\end{lemma}
 \noindent The proof Lemma \ref{lem_1} is straightforward by examining the Karush–Kuhn–Tucker (KKT) complementary slackness condition over the increasing and strictly concave objective function  $U_k(\cdot), \forall k$.

\subsection{Overall Intelligent Resource Management Algorithm}
\begin{algorithm}[t]
	\begin{algorithmic}[1]
 		\fontsize{9}{9}\selectfont
		\protect\caption{Intelligent Resource Management Algorithm for Solving JFCS Problem \eqref{eq:JFCS1},  compliant with O-RAN}		\label{alg_JFCS}
       \global\long\def\algorithmicrequire{\textbf{Initialization:}}	
      \REQUIRE Set $t=1$ and select a positive scaling  factor $\varphi$. Initialize $\boldsymbol{\beta}_k[1] = \frac{1}{|\mathcal{P}_k|}[1,\cdots,1]$ and all queues are set to be empty: $q_k^{i,j}[1_1]=0$ and $\hat{q}_k[1_1]=0, \forall (i,j), k$.
      
		\global\long\def\algorithmicrequire{\textbf{Main Loop:}}
		\REQUIRE 
		\FOR[/*\textit{Long-term  scale $t$}*/]{each frame $t=1,2,\cdots,T$}
	   \STATE \textbf{Flow-Split Distribution:} Given $\{\mathbf{q}[t-1],\mathbf{A}[t-1]\}$, CU splits data-flows of all UEs  based on the optimal flow-split decisions $\boldsymbol{\beta}^*[t]$ by solving L-SP at Non-RT RIC:
	     \[  \underset{\boldsymbol{\beta}_k[t]\in \mathscr{B}[t],\forall k}{\mathrm{max}} \   \sum_{k\in\mathcal{K}}\mathscr{L}_k[t].\]
	   \FOR[/*\textit{Short-term  scale $t_s$}*/]{each time-slot $t_s = tT_f + s$ with $s\in\{1,\cdots,T_f\}$}
	    
	    \STATE \textbf{Congestion Controller:} Given the queue-length vector $\hat{\mathbf{q}}[t_s]$, solve S-SP1 \eqref{eq:JFCSSSP1} to obtain the  optimal  congestion  control variables:
	    \[a_k^*[t_s] = \min\Bigl\{ U_k^{'-1}\bigl(\frac{\hat{q}_k[t_s]}{\varphi\tau} \bigr),A^{\max}\Bigr\}, \forall k.\]
	    
	    \STATE \textbf{Weighted Queue-Length-Based Scheduler:} Given the queue-length vector $\hat{\mathbf{q}}[t_s]$ and  the flow-split distribution $\boldsymbol{\beta}^*[t]$, each RU $(i,j)\in\mathcal{P}_{k}$ schedules the service rate $r_{k}^{i,j}(\mathbf{w}[t_s])$ for UE $k\in\mathcal{K}$ by solving S-SP2:
	    \[\underset{\mathbf{r}[t_s],\mathbf{w}[t_s]}{\mathrm{max}} \quad   \sum_{k\in\mathcal{K}} \frac{\hat{q}_k[t_s]}{\tau}r_k(\mathbf{w}[t_s]), \ \st\   \eqref{eq:JFCS1c},       \eqref{eq:JFCS1e_Relaxed}.\]
	    
	    \STATE \textbf{Queue-Length Updates:} Queue-Lengths are updated as
	       \begin{align}
	          q_k^{i,j}[t_{s+1}] &= \bigl[q_k^{i,j}[t_s] + \beta_k^{i,j}[t]A_k[t]\tau \nonumber\\
	          &\qquad\qquad - r_{k}^{i,j}(\mathbf{w}[t_s])\tau \bigr]^+, \ \forall k, (i,j)  \nonumber\\
	          \hat{q}_k[t_{s+1}] &= \bigl[\hat{q}_k[t_s] + a_k[t_s]\tau - r_k(\mathbf{w}[t_s])\tau \bigr]^+,\ \forall k\nonumber.
	       \end{align}
	    \STATE Set $s=s+1$
	    \ENDFOR
          \STATE Update $\{\mathbf{q}[t],\mathbf{A}[t]\}:= \{q_k^{i,j}[t], A_k[t]\}_{k, (i,j)}$ to Non-RT RIC.
           \STATE Set $t=t+1$
        \ENDFOR
\end{algorithmic}
\end{algorithm}
To solve problem \eqref{eq:JFCS3} in different time scales, we now decompose it into three subproblems. To do so, we consider a worst-case design by developing an upper bound of $\Delta L[t_s]$ for given $(\mathbf{q}[t_s],\hat{\mathbf{q}}[t_s])$. From the inequality $([x]^+)^2 \leq x^2$ and $(x+y)^2-x^2 = 2xy + y^2$, we have
\begin{align}
   & \Delta L^{\UB}[t_s] \triangleq \sum_{k\in\mathcal{K}}\sum_{(i,j)\in\mathcal{P}_k} \frac{q_k^{i,j}[t_s]}{\tau}\bigl(\beta_k^{i,j}[t]A_k[t] - r_{k}^{i,j}(\mathbf{w}[t_s]) \bigr) \nonumber\\
    &\quad+ \sum_{k\in\mathcal{K}}\frac{\hat{q}_k[t_s]}{\tau}\bigl(a_k[t_s] - r_k(\mathbf{w}[t_s])\bigr) + B[t_s] \geq \Delta L[t_s]
\end{align}
where $B[t_s]\triangleq \frac{1}{2}\sum_{k\in\mathcal{K}}\sum_{(i,j)\in\mathcal{P}_k} \bigl(\beta_k^{i,j}[t]A_k[t] - r_{k}^{i,j}(\mathbf{w}[t_s]) \bigr)^2+\frac{1}{2}\sum_{k\in\mathcal{K}}\bigl(a_k[t_s] - r_k(\mathbf{w}[t_s])\bigr)^2$ is the summation of the second moments of the arrival and service processes. Following \cite{neely2010stochastic} and \cite{VuTWC2019}, we consider that $B[t_s]$ is finite and bounded by $\bar{B}$ for all $t_s$, \textit{i.e.}, $\mathbb{E}\{B[t_s]\big|\mathbf{q}[t_s],\hat{\mathbf{q}}[t_s]\}$ $\leq \bar{B}$. As a result, problem \eqref{eq:JFCS3} is simplified to
\begin{subequations} \label{eq:JFCS4}
	\begin{IEEEeqnarray}{cl}
		\ &\underset{\boldsymbol{\beta}, \bar{\boldsymbol{a}}, \bar{\mathbf{r}}, \mathbf{w}}{\mathrm{max}} \quad   \varphi\sum_{k\in\mathcal{K}}\mathbb{E}\{U_k(a_k[t_s])\} 
		- \mathbb{E}\{\Delta L^{\UB}[t_s]\}
		\label{eq:JFCS4a} \\
		& \st  \quad\   \eqref{eq:JFCS1c},     \eqref{eq:JFCS1d},   \eqref{eq:JFCS1e_Relaxed}.               \label{eq:JFCS4b}
	\end{IEEEeqnarray}
\end{subequations}

\textbf{Long-term subproblem (L-SP):} The flow-split distribution subproblem at time-frame $t$ is given as 
\begin{align}\label{eq:JFCSLSP}
	\textbf{L-SP}:\	\underset{\boldsymbol{\beta}_k[t]\in \mathscr{B}[t],\forall k}{\mathrm{max}} \quad   \sum_{k\in\mathcal{K}}\mathscr{L}_k[t]               
\end{align}
where $\mathscr{L}_k[t] = \sum_{(i,j)\in\mathcal{P}_k} \frac{q_k^{i,j}[t_s]}{\tau}\bigl(r_{k}^{i,j}(\mathbf{w}[t_s]) - \beta_k^{i,j}[t]A_k[t]\bigr)$. Although problem \eqref{eq:JFCSLSP} is a linear program in $\boldsymbol{\beta}$, it cannot be solved directly by standard optimization techniques because  $A_k[t],\forall k$ are \textit{incompletely} known at the beginning of time-frame $t$.

\textbf{Short-term subproblems (S-SPs):} The   congestion  control subproblem   at time-slot $t_s$ is 
\begin{align}\label{eq:JFCSSSP1}
	\textbf{S-SP1}:\	\underset{\boldsymbol{a}[t_s]\geq 0}{\mathrm{max}} \quad   \sum_{k\in\mathcal{K}}\bigl(\varphi U_k(a_k[t_s]) - \frac{\hat{q}_k[t_s]}{\tau}a_k[t_s]\bigr)      
\end{align}
which is an unconstrained convex problem. The optimal solution of \eqref{eq:JFCSSSP1} exists and is unique that is $a_k^*[t_s] = U_k^{'-1}\bigl(\frac{\hat{q}_k[t_s]}{\varphi\tau} \bigr), \forall k$, where $U_k^{'-1}(\cdot)$ denotes the inverse function of the first derivation of $U_k(\cdot)$. Given the optimal solution $\boldsymbol{\beta}^*[t]$,
 the short-term  power control optimization subproblem (\textit{i.e.}, the weighted queue-length-based scheduling) at time-slot $t_s$ is given as
\begin{align}\label{eq:JFCSSSP2}
	\textbf{S-SP2}:	\underset{\mathbf{r}[t_s],\mathbf{w}[t_s]}{\mathrm{max}} \quad   \sum_{k\in\mathcal{K}} \frac{\hat{q}_k[t_s]}{\tau}r_k(\mathbf{w}[t_s]), \ \st\   \eqref{eq:JFCS1c},       \eqref{eq:JFCS1e_Relaxed}.             
\end{align}
The overall intelligent resource management algorithm for solving the JFCS problem \eqref{eq:JFCS1} is summarized in Algorithm \ref{alg_JFCS}, where the solutions of subproblems will be provided next.

\section{Proposed Algorithms for Solving Subproblems}\label{sec_RAAlgo}
We are now in a position to solve L-SP \eqref{eq:JFCSLSP} and S-SP2 \eqref{eq:JFCSSSP2} in different time scales. The optimality of the latter  depends heavily  on the optimal flow-split decisions, which often require a prior knowledge of the statistical information of all possible paths at Non-RT RIC. However, the assumption of complete information is unrealistic due to the dynamic environment and the data collected from the RAN layer being only updated to Non-RT RIC only on the long-term scale. In this work, at time-frame $t$ we aim to exploit historical system information  accumulated from the previous time-slot, which can be used to build the smoothed best response in maximizing the long-term utility for each data flow.

{\color{black}
\subsection{Reinforcement Learning Algorithm for Solving L-SP \eqref{eq:JFCSLSP}}
The flow-split decision $\boldsymbol{\beta}_k[t]$ in problem \eqref{eq:JFCSLSP} can be estimated separably by minimizing $\mathscr{L}_k[t]$. This implies that the larger the queue-length $q_k^{i,j}[t_s]$, the lower the flow-split decision value $\beta_k^{i,j}[t]$ to guarantee fairness among all RUs $(i,j)\in\mathcal{P}_k$ (\textit{i.e.}, to avoid large queue-lengths $q_k^{i,j}$ at some RUs in the next time-slot $t_{s+1}$).
Let us denote  
\[u_k^{i,j}[t]\triangleq  \frac{q_k^{i,j}[t_s]}{\tau}\bigl(r_{k}^{i,j}(\mathbf{w}[t_s]) - \beta_k^{i,j}[t]A_k[t]\bigr)\]
as the instantaneous utility observation of data-flow $k$ at time-frame $t$ when selecting path $(i,j)\in\mathcal{P}_k$. The total utility observation of data-flow $k$, denoted by $ u_k[t]$,  is thus 
\begin{equation}
u_k[t] = \sum_{(i,j)\in\mathcal{P}_k}u_k^{i,j}[t].
\end{equation}
However, it is unable to build a smoothed best response based on $u_k^{i,j}[t]$ as it is not revealed  at the beginning of time-frame $t$. 
Inspired by \cite{BennisTWC2013}, we denote $\hat{u}_k^{i,j}[t]$ as the estimated utility of data-flow $k$ at time-frame $t$ when selecting path $(i,j)$. In addition, the actual utility observed by data-flow $k$ at time-frame $t$, denoted by $\bar{u}_k[t]$, is given as $\bar{u}_k[t] = u_k[t-1]$, which is based on feedback from Near-RT RIC at time $t-1$. By initializing $\hat{u}_k^{i,j}[1] = 0$, the estimated utility of data-flow $k$ is updated for action $\mathbf{c}_k[t]= c_k^{i,j}[t]$ as follows:
\begin{align}\label{eq_utilityestimate}
    \hat{u}_k^{i,j}[t] =&\ \hat{u}_k^{i,j}[t-1]+ \eta_u[t]\mathbbm{1}_{\{\mathbf{c}_k[t] = c_k^{i,j}[t]\}}\bigl(\bar{u}_k[t] \nonumber\\
    &\qquad\qquad\qquad\ - \hat{u}_k^{i,j}[t-1]\bigr), \ \forall t>1
\end{align}
 where $\eta_u > 0$ is the decreasing step size (i.e. the learning rate), which is often decreased over time to guarantee convergence. Naturally, $\hat{u}_k^{i,j}[1]$ is initialized as $\hat{u}_k^{i,j}[1] = 0$ for $t=1$. The indicator function $\mathbbm{1}_{\{x=y\}}= 1$ (resp. 0) if the condition $x=y$ is true (resp. false).

Next, we denote $\hat{\boldsymbol{\theta}}_k[t]\triangleq[\hat{\theta}_k^{i,j}[t]]_{(i,j)\in\mathcal{P}_k}$ as the estimated regret vector of data-flow $k$, where each element is updated  for action $\mathbf{c}_k[t]= c_k^{i,j}[t]$ as
\begin{align}\label{eq_regretestimate}
    \hat{\theta}_k^{i,j}[t] =& \hat{\theta}_k^{i,j}[t-1] + \eta_{\theta}[t]\mathbbm{1}_{\{\mathbf{c}_k[t] = c_k^{i,j}[t]\}}\bigl(\bar{u}_k[t] \nonumber\\
    &\qquad\qquad- \hat{u}_k^{i,j}[t]-  \hat{\theta}_k^{i,j}[t-1]\bigr),\ \forall t>1
\end{align}
with $\hat{\theta}_k^{i,j}[1] = 0$ and $\eta_{\theta}[t]$ being the learning rate.  In order to achieve high  performance in the long term, the L-SP must balance  exploration and exploitation processes. 
We note that trying all possible actions  to choose the best paths (\textit{e.g.} the exhaustive exploration) can offer the highest payoff, but with the cost of slow convergence and even computationally prohibitive. During the exploitation process, playing an action associated with the highest estimated utility in \eqref{eq_utilityestimate} will likely result in a very sub-optimal solution. To make this tradeoff more efficient, let us define the best response function $\hat{\boldsymbol{\beta}}[t] = f(\hat{\boldsymbol{\theta}}[t])$ as
\begin{IEEEeqnarray}{rCl}\label{eq_bestresponse}\small
    f(\hat{\boldsymbol{\theta}}[t])  := \underset{\boldsymbol{\beta}_k[t]\in \mathscr{B}[t]}{\argmin}\Bigr\{h\bigr(\boldsymbol{\beta}[t]\bigl) - \lambda\sum_{k\in\mathcal{K}}\sum_{(i,j)\in\mathcal{P}_k}\beta_k^{i,j}[t]\hat{\theta}_k^{i,j}[t]\Bigl\}.\qquad
\end{IEEEeqnarray}
Here $\lambda$ is the so-called trade-off factor (a.k.a. Boltzmann temperature) and $h\bigr(\boldsymbol{\beta}[t]\bigl)$ denotes the regularization function. We note that when $\lambda \rightarrow 0$, it leads to uniform probabilities of all actions, \textit{i.e.} $\beta_k^{i,j}[t]=1/|\mathcal{P}_k|,\forall (i,j)\in\mathcal{P}_k$. For   $\lambda \rightarrow \infty$, the second term in \eqref{eq_bestresponse} will dominate the best response function and then the actions associated with the highest estimated regret will be selected \cite{BennisTWC2013}.

\textbf{Regularization function:} The regularization function allows to learn the best paths that maximize its own performance and stabilize the flow-split decisions. The solutions to problem \eqref{eq:JFCSLSP} lie in the unit simplex for each data-flow. Therefore, we adopt the Gibbs-Shannon entropy as the regularization function, \textit{i.e.} $h\bigr(\boldsymbol{\beta}[t]\bigl) = \sum_{k\in\mathcal{K}}\sum_{(i,j)\in\mathcal{P}_k}\beta_k^{i,j}[t]\ln\bigl( \beta_k^{i,j}[t]\bigr)$, which is $K$-strongly convex.
Substituting $h\bigr(\boldsymbol{\beta}[t]\bigl)$ into \eqref{eq_bestresponse}, we have
\begin{align}\label{eq_bestresponse2}
    f(\hat{\boldsymbol{\theta}}[t])  :=& \underset{\boldsymbol{\beta}_k[t]\in \mathscr{B}[t], \forall k}{\argmin}\Bigr\{\sum_{k\in\mathcal{K}}\sum_{(i,j)\in\mathcal{P}_k}\beta_k^{i,j}[t]\ln\bigl( \beta_k^{i,j}[t]\bigr) \nonumber\\
    &\qquad\qquad\quad - \lambda\sum_{k\in\mathcal{K}}\sum_{(i,j)\in\mathcal{P}_k}\beta_k^{i,j}[t]\hat{\theta}_k^{i,j}[t]\Bigl\}.
\end{align}
The function $f(\hat{\boldsymbol{\theta}}[t])$ is convex and separable for each $\beta_k^{i,j}[t]$. 
By solving the following equation
\begin{equation}\nonumber
\partial f(\hat{\boldsymbol{\theta}}[t])/\partial \beta_k^{i,j}[t]  = \ln\bigl( \beta_k^{i,j}[t]\bigr) + 1 - \lambda\hat{\theta}_k^{i,j}[t] = 0
\end{equation}
we have 
\[\beta_k^{i,j}[t] = f(\hat{\theta}_{k}^{i,j}[t]) = \exp{\bigl(\lambda\hat{\theta}_k^{i,j}[t]-1}\bigr).\]
To ensure $\sum_{(i,j)\in\mathcal{P}_k }\beta_k^{i,j}[t]=1, \forall k$ (\textit{i.e.} the unit simplex for data-flow $k$), we normalize $f_{k}^{i,j}(\hat{\boldsymbol{\theta}}_{k}[t])$ through the exponentiated mirror function as
\begin{align}
     f_{k}^{i,j}(\hat{\boldsymbol{\theta}}_{k}[t]) &= \frac{\exp{\bigl(\bigr[\lambda\hat{\theta}_k^{i,j}[t] -1\bigl]^+}\bigr)}{\sum_{(i',j')\in\mathcal{P}_k}\exp{\bigl(\bigr[\lambda\hat{\theta}_k^{i',j'}[t]-1\bigl]^+}\bigr)} \nonumber\\
     &= \frac{\exp{\bigl(\lambda\bigr[\hat{\theta}_k^{i,j}[t]\bigl]^+}\bigr)}{\sum_{(i',j')\in\mathcal{P}_k}\exp{\bigl(\lambda\bigr[\hat{\theta}_k^{i',j'}[t]\bigl]^+}\bigr)}.
\end{align}
As a result, the estimated value of each element of flow-split vector $\boldsymbol{\beta}_k[t]$ is updated for all actions with the regret as
\begin{align}\label{eq_betaestimate}
    \beta_k^{i,j}[t] =\beta_k^{i,j}[t-1] + \eta_{\beta}[t]\bigl(f_{k}^{i,j}(\hat{\boldsymbol{\theta}}_{k}[t])-\beta_k^{i,j}[t-1]\bigr)
\end{align}
for $t>1$, where $\boldsymbol{\beta}_k[1] = \frac{1}{|\mathcal{P}_k|}[1,\cdots,1]$ and $\eta_{\beta}[t]$ is the learning rate. The three-step reinforcement learning procedure includes \eqref{eq_utilityestimate}, \eqref{eq_regretestimate} and \eqref{eq_betaestimate}, which do not require expensive computations and projection to the feasible space.

\textbf{Convergence properties}:
The convergence conditions for the three-step reinforcement learning procedure are given as follows:
\begin{align}\label{eq_learningcondi}
  &\underset{T\to\infty}{\lim}\sum\nolimits_{t=1}^{T}\eta_u[t] = +\infty \ \&\ \underset{T\to\infty}{\lim}\sum\nolimits_{t=1}^{T}\eta_u^2[t] < +\infty\nonumber\\
  &\underset{T\to\infty}{\lim}\sum\nolimits_{t=1}^{T}\eta_\theta[t] = +\infty \ \&\ \underset{T\to\infty}{\lim}\sum\nolimits_{t=1}^{T}\eta_{\theta}^2[t] < +\infty\nonumber\\
  &\underset{T\to\infty}{\lim}\sum\nolimits_{t=1}^{T}\eta_{\beta}[t] = +\infty \ \&\ \underset{T\to\infty}{\lim}\sum\nolimits_{t=1}^{T}\eta_{\beta}^2[t] < +\infty\nonumber\\
  &\underset{t\to\infty}{\lim}\frac{\eta_\theta[t]}{\eta_u[t]} = 0 \ \&\ \underset{t\to\infty}{\lim}\frac{\eta_{\beta}[t]}{\eta_\theta[t]} = 0.
\end{align}
This implies that the learning rates must be decreased over time to guarantee the convergence of the proposed  three-step RL procedure. The detailed proof of a multiple-timescales RL algorithm can be found in \cite{BennisTWC2013,leslie2003convergent}. Following the same arguments as those in \cite{BennisTWC2013} and  the conditions in \eqref{eq_learningcondi}, the three-step RL procedure in \eqref{eq_utilityestimate}, \eqref{eq_regretestimate} and \eqref{eq_betaestimate} converges to the optimal solution with the positive trade-off factor $\lambda >0$, satisfying $\underset{t\to\infty}{\lim} \boldsymbol{\beta}_k[t] = \boldsymbol{\beta}_k^*,\ \forall k\in\mathcal{K}$.
}

\subsection{Proposed Algorithm for Solving  S-SP2 \eqref{eq:JFCSSSP2}}
 Given the optimal flow-split distribution of data-flow $k$, $\boldsymbol{\beta}_k^*[t]$, we denote by  $\mathcal{P}_k^*[t]$ the set of selected path states in time-frame $t$, which only includes $c_k^{i,j}[t]=1$ with $(i,j)\in\mathcal{P}_k$. 
In this section, we present two low-complexity transmission designs for $\mathbf{w}$, \textit{namely} maximum ratio transmission (MRT) and zero-forcing beamforming (ZFBF), and then develop low-complexity iterative algorithms for their solution.

\subsubsection{MRT-Based Transmission Design}
Each RU $(i,j)$ performs MRT beamforming (a.k.a. channel-mathched beamforming) using local CSI as $\mathbf{w}_{k}^{i,j}[t_s] = \frac{\sqrt{p_{k}^{i,j}[t_s]}}{\sqrt{\nu_{k}^{i,j}[t_s]}}\mathbf{h}_{k}^{i,j}[t_s],\ \forall (i,j)\in\mathcal{P}_k[t]$ and $k\in\mathcal{K}$, where $\nu_{k}^{i,j}[t_s] \triangleq \|\mathbf{h}_{k}^{i,j}[t_s]\|_2^2$ and $p_{k}^{i,j}[t_s]$ is the transmit power coefficient allocated to UE $k$ from RU $(i,j)$ in time-slot $t_s$. The corresponding SINR is rewritten as
\begin{align}
    \gamma_{k}^{i,j}(\mathbf{p}[t_s])=\frac{p_{k}^{i,j}[t_s]\nu_{k}^{i,j}[t_s]}{\Phi_k^{i,j}(\mathbf{p}[t_s])}
\end{align}
where $\Phi_k^{i,j}(\mathbf{p}[t_s]) \triangleq \sum_{(i',j')\in\mathcal{P}_k[t]\setminus\{(i,j)\}}p_{k}^{i',j'}[t_s]\nu_{k}^{i',j'}[t_s] +  \sum_{k'\in\mathcal{K}\setminus\{k\}}\sum_{(i,j)\in\mathcal{P}_{k'}[t]}p_{k'}^{i,j}[t_s]\frac{|(\mathbf{h}_{k}^{i,j}[t_s])^{\sfH}\mathbf{h}_{k'}^{i,j}[t_s]|^2}{\nu_{k'}^{i,j}[t_s]} +N_0$ is linear in $\mathbf{p}[t_s]\triangleq \bigl[p_{k}^{i,j}[t_s]\bigr]_{k\in\mathcal{K},(i,j)\in\mathcal{P}_k}$.
As a result, the  short-term power optimization problem \eqref{eq:JFCSSSP2} with MRT reduces to the following problem: 
\begin{subequations}\label{eq:JFCSSSP2_Equi}
\begin{IEEEeqnarray}{cl}
		\underset{\mathbf{p}[t_s]}{\mathrm{max}}& \quad   \sum_{k\in\mathcal{K}} \frac{\hat{q}_k[t_s]}{\tau}r_k(\mathbf{p}[t_s])\label{eq:JFCSSSP2_Equia}\\
		 \st&\quad  \bar{R}_k^{i,j}[t_s] \leq r_{k}^{i,j}(\mathbf{p}[t_s])\tau,\ \forall k, (i,j)\label{eq:JFCSSSP2_Equib}\\
		 &\quad \sum_{k\in\mathcal{K}}p_{k}^{i,j}[t_s] \leq P_{\max}^{i,j},\ \forall (i,j)\label{eq:JFCSSSP2_Equic}
\end{IEEEeqnarray}
\end{subequations}
where $r_k(\mathbf{p}[t_s]) = \sum_{(i,j)\in\mathcal{P}_k^*}r_{k}^{i,j }(\mathbf{p}[t_s])$ with $r_{k}^{i,j}(\mathbf{p}[t_s])\triangleq W\log_2\bigl(1 + \gamma_{k}^{i,j}(\mathbf{p}[t_s])\bigr)$, and $\bar{R}_k^{i,j}[t_s]\triangleq \sum_{\ell=1}^{t}\beta_k^{i,j}[\ell]\bar{A}_k\tau - (1-\epsilon_k)\bar{A}_k\bar{d}_k - \sum_{\ell=1}^{t_{s-1}}r_{k}^{i,j}(\mathbf{p}[\ell])\tau.$ 

Problem \eqref{eq:JFCSSSP2_Equi} is nonconvex due to the nonconcavity of $r_{k}^{i,j }(\mathbf{p}[t_s])$. We will now apply the inner approximation (IA) method  to effectively solve \eqref{eq:JFCSSSP2_Equi} in an iterative manner. Following from  inequality \eqref{eq_IAapproConcave} in Appendix \ref{app:DerivationofInequ} with $v=p_{k}^{i,j}[t_s]\|\mathbf{h}_{k}^{i,j}[t_s]\|_2^2$ and $z=\Phi_k^{i,j}(\mathbf{p}[t_s])$, the global concave lower bound of $r_{k}^{i,j}(\mathbf{p}[t_s])$ at the updated feasible point $\mathbf{p}^{(n)}[t_s]$ found at iteration $n$, denoted by ${r}_{k}^{i,j(n)}(\mathbf{p}[t_s];{\mathbf{p}}^{(n)}[t_s])$, is given as
\begin{align}
    r_{k}^{i,j}(\mathbf{p}[t_s]) &\geq  r_{k}^{i,j}({\mathbf{p}}^{(n)}[t_s]) - W\log_2e\biggl[ \gamma_{k}^{i,j}({\mathbf{p}}^{(n)}[t_s]) \nonumber \\ 
    &\ -  2\frac{\nu_{k}^{i,j}[t_s]\sqrt{{p}_{k}^{i,j(n)}[t_s]}\sqrt{p_{k}^{i,j}[t_s]}}{\Phi_k^{i,j}({\mathbf{p}}^{(n)}[t_s])} + \gamma_{k}^{i,j}({\mathbf{p}}^{(n)}[t_s])\nonumber\\
    &\ \times\frac{p_{k}^{i,j}[t_s]\nu_{k}^{i,j}[t_s] + \Phi_k^{i,j}(\mathbf{p}[t_s]) }{{p}_{k}^{i,j(n)}[t_s]\nu_{k}^{i,j}[t_s] + \Phi_k^{i,j}({\mathbf{p}}^{(n)}[t_s])} \biggl]\nonumber\\
    &:={r}_{k}^{i,j(n)}(\mathbf{p}[t_s];{\mathbf{p}}^{(n)}[t_s])\label{eq_Concaveofrate}
\end{align}
with ${r}_{k}^{i,j(n)}({\mathbf{p}}^{(n)}[t_s];{\mathbf{p}}^{(n)}[t_s]) = W\log_2\bigl(1 + \gamma_{k}^{i,j}(\mathbf{p}^{(n)}[t_s])\bigr)$. As a result, we successively  solve the following inner convex approximate program of \eqref{eq:JFCSSSP2_Equi} at iteration $n$:
\begin{subequations}\label{eq:JFCSSSP2_Convex}
\begin{IEEEeqnarray}{cl}
		\underset{\mathbf{p}[t_s]}{\mathrm{max}}& \   \sum_{k\in\mathcal{K}} \frac{\hat{q}_k[t_s]}{\tau}{r}_k^{(n)}(\mathbf{p}[t_s])\label{eq:JFCSSSP2_Convexa}\\
		 \st&\  \bar{R}_k^{i,j}[t_s] \leq {r}_{k}^{i,j(n)}(\mathbf{p}[t_s];{\mathbf{p}}^{(n)}[t_s])\tau,\ \forall k\in\mathcal{K}, (i,j)\label{eq:JFCSSSP2_Convexb}\qquad\\
		 &\ \sum_{k\in\mathcal{K}}p_{k}^{i,j}[t_s] \leq P_{\max}^{i,j},\ \forall (i,j)\label{eq:JFCSSSP2_Convexc}
\end{IEEEeqnarray}
\end{subequations}
and  update the feasible point ${\mathbf{p}}^{(n)}[t_s]$ until convergence, where ${r}_k^{(n)}(\mathbf{p}[t_s])=\sum_{(i,j)\in\mathcal{P}_k^*}r_{k}^{i,j(n)}(\mathbf{p}[t_s];$ ${\mathbf{p}}^{(n)}[t_s])$. The proposed iterative procedure  to solve \eqref{eq:JFCSSSP2} is summarized in Algorithm \ref{alg_SSP2}. An initial feasible  value for $\mathbf{p}^{(0)}[t_s]$ to start Algorithm \ref{alg_SSP2} is easily found by successively solving the following simple convex program:
\begin{subequations}\label{eq:JFCSSSP2_ConvexInf}
\begin{IEEEeqnarray}{cl}
		\underset{\mathbf{p}[t_s]}{\mathrm{max}}& \  \varrho \triangleq \underset{\forall k,(i,j)}{\min}\bigl\{{r}_{k}^{i,j(n)}(\mathbf{p}[t_s];{\mathbf{p}}^{(n)}[t_s])\tau - \bar{R}_k^{i,j}[t_s]\bigr\}\qquad\label{eq:JFCSSSP2_ConvexInfa}\\
		 \st&\quad \sum_{k\in\mathcal{K}}p_{k}^{i,j}[t_s] \leq P_{\max}^{i,j},\ \forall (i,j)\label{eq:JFCSSSP2_ConvexInfb}
\end{IEEEeqnarray}
\end{subequations}
until reaching $\varrho > 0$.

\begin{algorithm}[t]
\begin{algorithmic}[1]
\fontsize{10}{10}\selectfont
\protect\caption{Proposed Iterative Algorithm for Solving  \eqref{eq:JFCSSSP2} with MRT-Based Transmission Design}
\label{alg_SSP2}
\global\long\def\algorithmicrequire{\textbf{Initialization:}}
\REQUIRE  Set $n:=1$ and  generate an initial feasible value for   $\mathbf{p}^{(0)}[t_s]$ to constraints in \eqref{eq:JFCSSSP2_Convex}
\REPEAT
\STATE Solve  \eqref{eq:JFCSSSP2_Convex} to obtain the optimal transmission power $\mathbf{p}^{*}[t_s]$

\STATE Update\ \ ${\mathbf{p}}^{(n)}[t_s] := \mathbf{p}^{*}[t_s]$
\STATE Set $n:=n+1$
\UNTIL Convergence\\
\STATE{\textbf{Output:}} $\mathbf{p}^{*}[t_s]={\mathbf{p}}^{(n)}[t_s]$ and $\mathbf{w}_{k}^{i,j,*}[t_s] = \frac{\sqrt{p_{k}^{i,j*}[t_s]}}{\sqrt{\nu_{k}^{i,j}[t_s]}}\mathbf{h}_{k}^{i,j}[t_s],\ \forall k, (i,j)$.
\end{algorithmic} \end{algorithm}

\textit{Convergence and complexity analysis:} The convergence of an IA-based  algorithm  is already provided in \cite{Beck:JGO:10}. In particular, Algorithm \ref{alg_SSP2} generates an improved solution after each iteration, which   converges to at least a local optimal solution of \eqref{eq:JFCSSSP2} when $n \rightarrow\infty$. The worst-case of per-iteration complexity of Algorithm \ref{alg_SSP2} is $\mathcal{O}\bigl(\sqrt{c}(v)^3 \bigr)$ by the interior-point method \cite[Chapter 6]{Ben:2001}, where $c=KJ + J$ and $v=KJ$ are the numbers of linear constraints and scalar variables, respectively.

\subsubsection{ZFBF-Based Transmission Design}
To make ZFBF efficient and feasible, the number of antennas of each RU $(i,j)$ is required to be larger than the number of UEs, i.e. $M_{i,j} > K,\ \forall (i,j)\in\mathcal{J}$, to cancel the inter-user interference transmitted by this RU. In addition, the system bandwidth is equally allocated to each RU $(i,j)$, i.e. $W^{i,j}=W/J$, to completely remove the intra-user interference and interference caused by other RUs. Under the proposed ZFBF technique, beamformer $\mathbf{w}_k^{i,j}[t_s]$ at RU $(i,j)$ is designed to satisfy $(\mathbf{h}_{k'}^{i,j}[t_s])^{\sfH}\mathbf{w}_k^{i,j}[t_s]=0, \forall k'\in\mathcal{K}\setminus\{k\}$. We denote by $\mathbf{H}_{-k}^{i,j}[t_s]\triangleq \bigr[\mathbf{h}_{1}^{i,j}[t_s]\cdots\mathbf{h}_{k-1}^{i,j}[t_s]\ \mathbf{h}_{k+1}^{i,j}[t_s]\cdots\mathbf{h}_{K}^{i,j}[t_s]\bigr]\in\mathbb{C}^{M\times (K-1)}$  the channel matrix from RU ($i,j$) to UEs, except UE $k$. Let $\mathbf{V}_{k}^{i,j}[t_s]\in\mathbb{C}^{M_{i,j}\times(M_{i,j}-K+1)}$ be the null space of $(\mathbf{H}_{-k}^{i,j}[t_s])^{\sfH}$. We can then write  $\mathbf{w}_k^{i,j}[t_s] = \mathbf{V}_{k}^{i,j}[t_s]\tilde{\mathbf{w}}_k^{i,j}[t_s]$, where $\tilde{\mathbf{w}}_k^{i,j}[t_s]\in\mathbb{C}^{(M_{i,j}-K+1)\times 1}, \forall k,(i,j)$ are the solutions to the ZFBF-based problem. By defining $\tilde{\nu}_k^{i,j}[t_s]\triangleq\|(\tilde{\mathbf{h}}_k^{i,j}[t_s])^{\sfH}\|_2^2$ with $\tilde{\mathbf{h}}_k^{i,j}[t_s]\triangleq(\mathbf{h}_k^{i,j}[t_s])^{\sfH}\mathbf{V}_{k}^{i,j}[t_s]\in\mathbb{C}^{1\times (M_{i,j}-K+1)}$, we can equivalently express $\tilde{\mathbf{w}}_k^{i,j}[t_s]$ as $\tilde{\mathbf{w}}_k^{i,j}[t_s]=\sqrt{\tilde{p}_{k}^{i,j}[t_s]}\frac{(\tilde{\mathbf{h}}_k^{i,j}[t_s])^{\sfH}}{\sqrt{\tilde{\nu}_k^{i,j}[t_s]}}$, where $\tilde{\mathbf{p}}[t_s]\triangleq\bigl[\tilde{p}_{k}^{i,j}[t_s]\bigr]_{k,(i,j)\in\mathcal{P}_k}$ are the solutions to the following problem:
\begin{subequations}\label{eq:JFCSSSP2_ZFBF}
\begin{IEEEeqnarray}{cl}
		\underset{\tilde{\mathbf{p}}[t_s]}{\mathrm{max}}& \quad   \sum_{k\in\mathcal{K}} \frac{\hat{q}_k[t_s]}{\tau}{r}_k(\tilde{p}_{k}^{i,j}[t_s])\label{eq:JFCSSSP2_ZFBFa}\\
		 \st&\quad  \bar{R}_k^{i,j}[t_s] \leq r_{k}^{i,j}(\tilde{p}_{k}^{i,j}[t_s])\tau,\ \forall k, (i,j)\label{eq:JFCSSSP2_ZFBFb}\\
		 &\quad \sum_{k\in\mathcal{K}}\tilde{p}_{k}^{i,j}[t_s] \leq P_{\max}^{i,j},\ \forall (i,j)\label{eq:JFCSSSP2_ZFBFc}
\end{IEEEeqnarray}
\end{subequations}
where $r_{k}^{i,j}(\tilde{p}_{k}^{i,j}[t_s])\triangleq W^{i,j}\log_2\bigl(1 + \frac{\tilde{p}_{k}^{i,j}[t_s]\tilde{\nu}_k^{i,j}[t_s]}{N_0}\bigr)$. The function $r_{k}^{i,j}(\tilde{p}_{k}^{i,j}[t_s])$ is concave in $\tilde{p}_{k}^{i,j}[t_s]$, leading to the convexity of problem \eqref{eq:JFCSSSP2_ZFBF}. From \eqref{eq:JFCSSSP2_ZFBFb}, one can show that $\tilde{p}_{k}^{i,j}[t_s] \geq \tilde{p}_{k,\min}^{i,j}[t_s]:=\frac{N_0}{\tilde{\nu}_k^{i,j}[t_s]} 2^{\frac{\bar{R}_k^{i,j}[t_s]}{W^{i,j}\tau} -1}$. We now develop an efficient method to solve \eqref{eq:JFCSSSP2_ZFBF} by formulating the partial Lagrangian as
\begin{align}
    L(\tilde{\mathbf{p}}[t_s],\boldsymbol{\mu}) =&\ \sum_{k\in\mathcal{K}} \frac{\hat{q}_k[t_s]}{\tau}{r}_k(\tilde{p}_{k}^{i,j}[t_s]) \nonumber\\
    &+ \sum_{(i,j)\in\mathcal{J}}\mu_{i,j}\bigl(P_{\max}^{i,j} - \sum_{k\in\mathcal{K}}\tilde{p}_{k}^{i,j}[t_s] \bigr)
\end{align}
where $\boldsymbol{\mu}\triangleq\{\mu_{i,j}\geq 0\}_{(i,j)\in\mathcal{J}}$ are the Lagrange multipliers of constraint \eqref{eq:JFCSSSP2_ZFBFc}. The dual function can be written as $g(\boldsymbol{\mu}) = \underset{\tilde{\mathbf{p}}[t_s] \geq 0}{\max}\{L(\tilde{\mathbf{p}}[t_s],\boldsymbol{\mu})|\tilde{p}_{k}^{i,j}[t_s] \geq \tilde{p}_{k,\min}^{i,j}[t_s], \forall k,$ $(i,j)\}$. We note that $L(\tilde{\mathbf{p}}[t_s],\boldsymbol{\mu})$ is separable with respect to $\tilde{p}_{k}^{i,j}[t_s]$. Thus, by solving
\begin{align}
    \tilde{p}_{k}^{i,j*}[t_s] = &\underset{\tilde{p}_{k}^{i,j}[t_s]\geq \tilde{p}_{k,\min}^{i,j}[t_s]}{\argmax}\Bigr\{\frac{\hat{q}_k[t_s]}{\tau} W\log_2\Bigl(1 + \frac{\tilde{p}_{k}^{i,j}[t_s]\tilde{\nu}_k^{i,j}[t_s]}{N_0}\Bigr) \nonumber\\
    &- \mu_{i,j}\tilde{p}_{k}^{i,j}[t_s] \Bigl\}
\end{align}
for a given $\mu_{i,j}$, the optimal solution to $\tilde{p}_{k}^{i,j}[t_s] $ is given as
\begin{align}\label{eq_powerZFBF}
    \tilde{p}_{k}^{i,j*}[t_s]=\max\Bigr\{\tilde{p}_{k,\min}^{i,j}[t_s], \frac{\hat{q}_k[t_s] W^{i,j}}{\tau\mu_{i,j}\ln 2} -\frac{N_0}{\tilde{\nu}_k^{i,j}[t_s]}\Bigl\}.
\end{align}
\begin{algorithm}[t]
\begin{algorithmic}[1]
\fontsize{10}{10}\selectfont
\protect\caption{Proposed Low-Complexity Algorithm for Solving  \eqref{eq:JFCSSSP2} with ZFBF-Based Transmission Design}
\label{alg_SSP3}
\global\long\def\algorithmicrequire{\textbf{Initialization:}}
\REQUIRE  Set $n:=1$ and  generate initial values $\underline{\mu}_{i,j}=0$ and $\overline{\mu}_{i,j}=+\infty, \forall (i,j)\in\mathcal{J}$
\FOR{each RU $(i,j)\in\mathcal{J}$ \textit{in parallel}}
\REPEAT
\STATE Compute $\mu_{i,j}^{(n)}=(\underline{\mu}_{i,j}+\overline{\mu}_{i,j})/2$ and $\tilde{p}_{k}^{i,j(n)}[t_s]$ as\qquad  in \eqref{eq_powerZFBF}
\IF{\ $\sum_{k\in\mathcal{K}}\tilde{p}_{k}^{i,j(n)}[t_s] - P_{\max}^{i,j} \leq 0$}
   \STATE Compute\ $\mu'_{i,j}=(\underline{\mu}_{i,j}+\overline{\mu}_{i,j})/2$\ and\ update\ $\overline{\mu}_{i,j}:=\mu'_{i,j}$
\ELSE 
 \STATE Update\ $\mu'_{i,j}=(\underline{\mu}_{i,j}+\overline{\mu}_{i,j})/2$\ and\ update\ $\underline{\mu}_{i,j}:=\mu'_{i,j}$
\ENDIF
 \STATE Set $n:=n+1$
\UNTIL $\overline\mu_{i,j}-\underline\mu_{i,j} \leq \delta$\   \{/*\textit{Satisfying a given accuracy level*}/\}
\ENDFOR
\STATE{\textbf{Output:}} $\mu_{i,j}^*=\mu_{i,j}^{(n)}$, $\tilde{p}_{k}^{i,j*}[t_s]=\max\Bigr\{\tilde{p}_{k,\min}^{i,j}[t_s], \frac{\hat{q}_k[t_s]W^{i,j}}{\tau\mu_{i,j}^*\ln 2} -\frac{N_0}{\tilde{\nu}_k^{i,j}[t_s]}\Bigl\}$ and $\mathbf{w}_k^{i,j,*}[t_s]=\frac{\sqrt{\tilde{p}_{k}^{i,j*}[t_s]}}{\sqrt{\tilde{\nu}_k^{i,j}[t_s]}}\mathbf{V}_{k}^{i,j}[t_s](\tilde{\mathbf{h}}_k^{i,j}[t_s])^{\sfH},\  \forall k,(i,j)$.
\end{algorithmic} \end{algorithm}
The optimal Lagrange multiplier $\mu_{i,j}$ is efficiently found by applying a bisection search method between $\underline{\mu}_{i,j}=0$ and a sufficiently large $\overline{\mu}_{i,j}$. An efficient algorithm for solving  \eqref{eq:JFCSSSP2} with ZFBF is summarized in Algorithm \ref{alg_SSP3}, which does not rely on existing convex optimization solvers. 

\subsection{O-RAN-based Implementation of Algorithm \ref{alg_JFCS}}

\begin{figure}[t]
	\centering
\includegraphics[width=0.85\columnwidth,trim={0cm 0.0cm 0cm 0.0cm}]{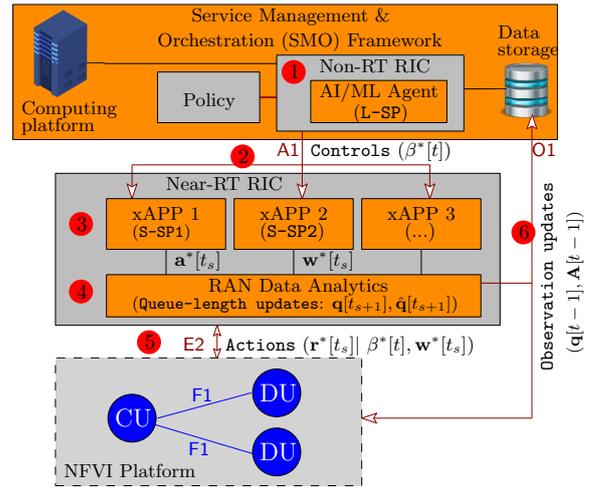}
	\caption{O-RAN Alliance reference architecture for implementing the proposed JFCS management scheme at time-frame $t$. }
\label{fig:ORANIplementation}
\end{figure}

{\color{black}
Fig. \ref{fig:ORANIplementation} illustrates the key steps for implementing the proposed JFCS management scheme at time-frame $t$ in the O-RAN architecture.
\renewcommand{\labelenumi}{(\arabic{enumi})}
\begin{enumerate}
\item[\circled{\cnt}] At the beginning of time-frame $t>1$, the three-step RL procedure for solving L-SP is carried out at Non-RT RIC based on the collected RAN data in SMO. The collected data include performance/observation and resource updates from CU, DU, RU and Near-RT RIC to SMO. For $t=1$, the flow-split decisions are initialized as $\boldsymbol{\beta}_k[1] = \frac{1}{|\mathcal{P}_k^{(1)}|}[1,\cdots,1], \forall k$ where $\mathcal{P}_k^{(1)}$ is the set of RUs in the feasible communication range of UE $k$.

\item[\circled{\cnt}] The optimal flow-split decisions $\boldsymbol{\beta}^*[t]$ are sent to Near-RT RIC via  A1  (the standardized open interface) for real deployment.

\item[\circled{\cnt}] Near-RT RIC hosts xApps (i.e. third party
applications) which communicate with CU/DU through the E2 interface. Given $\boldsymbol{\beta}^*[t]$, xAPPs deployed in Near-RT RIC control congestion and optimize RAN resources and functions in each time-slot $t_s$ of time-frame $t$ by solving S-SP1 and S-SP2 to obtain the optimal solutions of congestion control $\boldsymbol{a}^*[t_s]$  and beamformer $\mathbf{w}^*[t_s]$.

\item[\circled{\cnt}] Subsequently, the RAN Data Analytic component in Near-RT RIC updates queue-lengths  as in Step 6 of Algorithm \ref{alg_JFCS}. The updated queue-lengths are sent back to SMO through the O1 interface for periodic reporting.

\item[\circled{\cnt}] Given $\boldsymbol{\beta}^*[t]$ and $\mathbf{w}^*[t_s]$, the optimal service rate $\mathbf{r}(\mathbf{w}^*[t_s])$ is scheduled and applied to CU and DUs through  the E2 interface.

\item[\circled{\cnt}] After $T_f$ time-slots in the short-term scale $t_s$, performance and observations (e.g. $\mathbf{q}[t-1], \mathbf{A}[t-1]$) are updated to SMO through the O1 interface to re-estimate the flow-split decision $\boldsymbol{\beta}^*[t+1]$.
\end{enumerate}
}

\section{Performance Analysis of The JFCS Framework}\label{sec_PerfAnaly}
In this section, we analyze the main theoretical performance results of Algorithm \ref{alg_JFCS} and discuss their key insights, followed by concrete  proofs of the theorems. 
\begin{assumption}\label{assump_2} To facilitate the analysis, we make the following additional assumptions.
\begin{itemize}
    \item Under the limited transmit power budget at RUs, the achievable rate of UE $k$  is upper bounded by $r^{\max} >0$, \textit{i.e.}, $r_k(\mathbf{w}[t_s])\leq r^{\max},\ \forall k, t_s$.  
    \item The congestion control variable $a_k[t_s]$ satisfies the condition $\bbE\{a_k^2[t_s]\} \leq A_1^{\max}$, where $A_1^{\max}$ is a sufficiently large positive constant \cite{LiuJSAC2017}.
\end{itemize}
\end{assumption}


\begin{theorem}[Bounding the mean divergence of the auxiliary queue-length]\label{Theo_1} For a given scaling factor $\varphi$, let $\hat{\mathbf{q}}^{\infty}_{(\varphi)}$ and $\hat{\mathbf{q}}^*_{(\varphi)}$ be the steady-state and optimal queue-lengths,  respectively. From Assumptions \ref{assp:1} and \ref{assump_2}, the expected upper bound of the divergence of $\hat{\mathbf{q}}^{\infty}_{(\varphi)}$ from $\hat{\mathbf{q}}^*_{(\varphi)}$ is given as
\begin{equation}\label{eq_theo1eq}
 \mathbb{E}\{\|\hat{\mathbf{q}}^{\infty}_{(\varphi)} - \hat{\mathbf{q}}^*_{(\varphi)}\|_2\} \leq \mathsf{C}_1\sqrt{\varphi} = \mathcal{O}(\sqrt{\varphi})
\end{equation}
where $\mathsf{C}_1 \triangleq \sqrt{\frac{K\tau^2\Psi}{2}\bigl(A_1^{\max} + (r^{\max})^2\bigr)}$ is a positive constant.
\end{theorem}
\noindent The proof is detailed in Appendix \ref{App_B}. Theorem \ref{Theo_1}
implies that the divergence of the steady-state queue-length is bounded by $\mathcal{O}(\sqrt{\varphi})$. In particular, the smaller the value of $\varphi$, the less the divergence of $\hat{\mathbf{q}}^{\infty}_{(\varphi)}$. However, a small $\varphi $ will also result in a small congestion control rate and a faster convergence. When $\varphi$ is large, a better congestion control rate is achieved but with the cost of  larger steady-state queue-length divergence (i.e. larger delay and slower convergence). Hence, there exists an appropriate value of $\varphi$ to make this tradeoff more efficient. Theorem \ref{Theo_1} will immediately lead to the following  result.

\begin{corollary}[Queue-stability]\label{corollary1}
Given a scaling factor $\varphi$ and $\mathsf{C}_1$ in \eqref{eq_theo1eq}, the steady-state total queue-length remains  finite and scales as $\mathcal{O}(\varphi) + \mathcal{O}(\sqrt{\varphi})$, \textit{i.e.} \begin{align}\label{steady-statebound}
    \underset{t_s\to\infty}{\limsup}\ \mathbb{E}\{\|\hat{\mathbf{q}}_{(\varphi)}[t_s]\|_1\} &\leq \tau\Psi K A^{\max} \varphi + \sqrt{K}\mathsf{C}_1\sqrt{\varphi}  \nonumber\\
    &= \mathcal{O}(\varphi) + \mathcal{O}(\sqrt{\varphi}).
\end{align}
\end{corollary}
\begin{proof}The proof of \eqref{steady-statebound} is straightforward by noticing the fact that $\underset{t_s\to\infty}{\limsup}\ \mathbb{E}\{\|\hat{\mathbf{q}}_{(\varphi)}[t_s]\|_1\} =  \mathbb{E}\{\|\hat{\mathbf{q}}_{(\varphi)}^{\infty}\|_1\}=\mathbb{E}\{\|\hat{\mathbf{q}}_{(\varphi)}^{\infty}-\hat{\mathbf{q}}_{(\varphi)}^{*}\|_1 + \|\hat{\mathbf{q}}_{(\varphi)}^{*}\|_1 \}$. Applying the inequality  $\|\mathbf{x}\|_1\leq \sqrt{K}\|\mathbf{x}\|_2$ for any $\mathbf{x}\in\mathbb{R}_+^K$ yields: $\|\hat{\mathbf{q}}_{(\varphi)}^{\infty}-\hat{\mathbf{q}}_{(\varphi)}^{*}\|_1 + \|\hat{\mathbf{q}}_{(\varphi)}^{*}\|_1 \leq \sqrt{K}\|\hat{\mathbf{q}}_{(\varphi)}^{\infty}-\hat{\mathbf{q}}_{(\varphi)}^{*}\|_2 + \|\hat{\mathbf{q}}_{(\varphi)}^{*}\|_1$. From \eqref{eq_theo1eq} and Step 4 of Algorithm \ref{alg_JFCS}, it follows that
\begin{align}
    &\underset{t_s\to\infty}{\limsup}\ \mathbb{E}\{\|\hat{\mathbf{q}}_{(\varphi)}[t_s]\|_1\} \nonumber\\
    &\leq \sqrt{K}\mathbb{E}\{\|\hat{\mathbf{q}}_{(\varphi)}^{\infty}-\hat{\mathbf{q}}_{(\varphi)}^{*}\|_2\} + \|\hat{\mathbf{q}}_{(\varphi)}^{*}\|_1\nonumber\\
    &\leq \sqrt{K}\mathsf{C}_1\sqrt{\varphi} + \tau \sum_{k\in\mathcal{K}}U_k^{'}\bigl(a_k^*\bigr) \varphi \leq \sqrt{K}\mathsf{C}_1\sqrt{\varphi} + \tau\Psi \sum_{k\in\mathcal{K}}a_k^* \varphi\ \nonumber\\
    &\leq \sqrt{K}\mathsf{C}_1\sqrt{\varphi} + \tau\Psi K A^{\max} \varphi\ (\text{due to}\ a_k^* \leq A^{\max},\forall k)
\end{align}
showing \eqref{steady-statebound}.\end{proof}

 Let $\bfa^{\infty}_{(\varphi)}\triangleq [a^{\infty}_{(\varphi),k}]^{\sfT}_{k\in\mathcal{K}}$ with $a^{\infty}_{(\varphi),k} = \mathbb{E}\bigl\{\min\{ A^{\max},$ $ U_k^{'-1}\bigl(\frac{\hat{q}_{(\varphi),k}^{\infty}}{\varphi\tau} \bigr)\}\bigl\}$ be the mean steady-state congestion control rate vector. We also denote by $U(\bfa)\triangleq \sum_{k\in\mathcal{K}}U_k(a_k)$ the total utility function of problem \eqref{eq:JFCS2}. The utility-optimality of Algorithm \ref{alg_JFCS} is stated by the following theorem, whose proof is given in Appendix \ref{App_C}.

\begin{theorem}[Optimality]\label{Theo_2} Given a scaling factor $\varphi$, Algorithm \ref{alg_JFCS} produces the mean steady-state congestion control rate vector $\bfa^{\infty}_{(\varphi)}$, satisfying 
\begin{align}\label{eq_theo2a}
    \|\bfa^{\infty}_{(\varphi)} - \bfa^{*}\|_2 \leq \mathsf{C}_2  \frac{1}{\sqrt{\varphi}} = \mathcal{O}(1/\sqrt{\varphi})
    \end{align}
where  $\mathsf{C}_2 \triangleq \frac{\mathsf{C}_1}{\psi\tau} = \sqrt{\frac{K\Psi}{2\psi}\bigl(A_1^{\max} + (r^{\max})^2\bigr)}$. Therefore, the optimal network utility maximization is bounded as
\begin{align}\label{eq_theo2b}
  U(\bfa^*) - \mathsf{C}_3\frac{1}{\varphi} =   U(\bfa^*) - \mathcal{O}(1/\varphi)\leq  U(\bfa^{\infty}_{(\varphi)}) 
    \end{align}
where $\mathsf{C}_3 \triangleq \frac{\Psi\mathsf{C}_1^2}{2\psi^2\tau^2} = \frac{K\Psi^2}{4\psi}\bigl(A_1^{\max} + (r^{\max})^2\bigr)$.    
\end{theorem}
The analytical results in Theorem \ref{Theo_2} show that the divergence of the steady-state  congestion control rate vector $\bfa^{\infty}_{(\varphi)}$ from $\bfa^*$ scales as $\mathcal{O}(1/\sqrt{\varphi})$, which is the same as  in \cite{LiuJSAC2017,EryilmazToN2007}. The utility-optimality gap can be reduced by increasing $\varphi$, but this will also lead to a larger steady-state queue-length divergence.

\section{Numerical Results}\label{sec_NumericalResults}
In this section, we first present simulation setup and parameters in Section \ref{sec_NumericalResults_A} and then provide numerical results of Algorithm \ref{alg_JFCS} in Section \ref{sec_NumericalResults_B}. The results and performance comparison over existing schemes  will be provided in Section \ref{sec_NumericalResults_C}.

\subsection{Simulation Setups and Parameters}\label{sec_NumericalResults_A}
\begin{figure}[!ht]
	\centering
	\includegraphics[width=1.1\columnwidth,trim={0cm 0.0cm 0cm 0.0cm}]{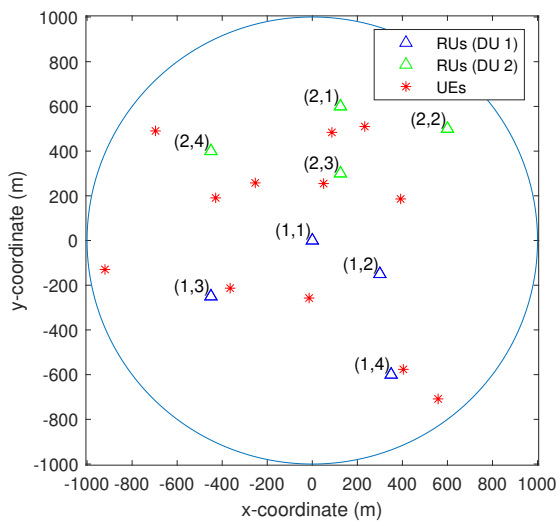}
	\caption{A system topology with $J=8$ RUs and $K=12$ UEs.}
	\label{fig:Layout}
\end{figure}

  \begin{table}[t]
	\centering  
	\captionof{table}{Simulation Parameters}
	\label{tab:Simulationparameter}
 	\scalebox{0.9}{
		\begin{tabular}{l|l}
			\hline
			Parameter & Value \\
			\hline\hline
			System bandwidth, $W$ & 20 MHz \\
 			Number of RUs, $J$ & 8\\
 			Number of UEs, $K$ & 12\\
			Number of antennas at RUs $(i,j)$, $M_{i,j} \equiv  M$ & 16\\
 			RUs' height  & 10 m\\
 			UEs' antenna  altitude  & 1.5 m\\
			Power budget at RU $(i,j)$, $P_{\max}^{i,j}\equiv P_{\max}$ & 43 dBm\\
			Noise figure, $\textsf{NF}$ & 9 dB\\
			Maximum average delay, $\bar{d}_k\equiv d$ & 10 ms\\
			Require reliable communication, $\epsilon_k\equiv \epsilon$  & 0.95\\
 			Number of frames, $T$ & 10000\\
 			Number of time-slots per frame, $T_f$ & 10\\
 			Duration of one frame, $T_c$ & 10 ms\\
			Duration of one time-slot, $\tau$ & 1 ms\\
			Trade-off factor (Boltzmann temperature), $\lambda$ & 0.3\\
		\hline		  				
		\end{tabular}
	}
\end{table}

We consider a system topology given in Fig.~\ref{fig:Layout}, including 8 RUs and $12$ UEs located within a circle of 1-km radius. There are two DUs, each connected to 4 RUs.  RUs are uniformly distributed in the area, while those of UEs are randomly located in each time-frame $t.$ The large-scale fading coefficient $\xi[t] \in \{\xi_{k}^{i,j}[t] \}_{\forall (i,j),k}$ is modeled as the three-slope path loss model  \cite{Ngo:TWC:Mar2017}, such as $\xi[t] = \xi_0 - 35 \log_{10}(d[t]) + 20c_0\log_{10}(d/d_0)+15c_1\log_{10}(d/d_1)$
where $\xi_0 = -140.7 + \mathsf{SF}$ dB, $d_0 = 10$ m,  $d_1 = 50$ m, and $d$ is the distance between an RU and a UE; here $c_i=\max\{0,\frac{d_i-d}{|d_i-d|}\}$ with $i\in\{0,1\}$ and $\mathsf{SF} \sim \mathcal{CN}(0, \sigma_{\mathsf{SF}})$ denotes the shadowing factor with $\sigma_{\mathsf{SF}}=8$ dB. The Rician factor $\kappa[t] \in \{\kappa_k^{i,j}[t]\}_{\forall (i,j),k}$ is given as $\kappa = P_{\mathsf{LoS}}(d[t])/ \big(1-P_{\mathsf{LoS}}(d[t])\big)$, where the LoS probability  follows the 3GPP–UMa model as $P_{\mathsf{LoS}}(d[t]) = \min \left( \frac{18}{d[t]},1 \right) \big( 1 - \exp(-\frac{d[t]}{36}) \big) + \exp(-\frac{d[t]}{36})$     \cite{jafari2015study}. We consider uniform linear arrays with half-wavelength distances between array elements to model the LoS channels at RUs. The array response vector is generated as $\bar{\mathbf{h}}^{i,j}_{k}[t] = \va(\phi^{i,j}_{k}[t])$, where each element $m$ is given as $\big[\va(\phi^{i,j}_{k}[t])\big]_{m} = \exp\bigr(j \pi (m-1) \sin \phi^{i,j}_{k}[t]\bigr)$ with $\phi^{i,j}_{k}[t]\in [-\pi / 2, \pi / 2)$ being the angle-of-departure (AoD) at RU $(i,j)$. The noise power is modeled as $N_0 = -170 + 10 \log_{10} (W) + \mathsf{NF}$ dBm, where $\mathsf{NF}=9$ dB denotes the noise figure.
  
We run Algorithm \ref{alg_JFCS} over $T=10000$ frames,  each consists of $T_f = 10$ time-slots (subframes) and has duration of $T_c=10$ ms, followed by 5G NR Frame structure \cite{3GPP_Release_15}.  In each time-frame $t$, UE $k$ is served by a subset of four RUs. To illustrate the heterogeneity of UEs, we assume that the arrival rate $A_k[t]$ is uniformly distributed in [1,\ 3] Gbps. The step sizes (learning rates) are set to decrease after each frame as $\eta_u[t]=1/(t+1)^{0.51}$, $\eta_\theta[t]=1/(t+1)^{0.55}$ and $\eta_\beta[t]=1/(t+1)^{0.6}$ \cite{SamarakoonTWC13}. We adopt the proportional fairness metric to model the utility function as: $U_k(r_k)=\log(0.001+r_k), \forall k$ \cite{XiaojunJSAC06}. The key parameters are summarized in Table \ref{tab:Simulationparameter} for ease of cross-referencing, followed by studies in \cite{jafari2015study,SamarakoonTWC13,3GPP_Release_15,Ngo:TWC:Mar2017,BennisTWC2013}. In the following figures, results are averaged over the last 6000 frames.

\textbf{Benchmark schemes:} To demonstrate the benefits of the proposed JFCS algorithm, we consider the following three benchmark schemes:
\begin{itemize}
    \item ``NUM with fixed resource allocation (NUM-FRA)" \cite{StaiCOMML14}: Under Algorithm \ref{alg_JFCS}, RUs allocate power equally to UEs. 
    \item ``NUM with equal flow-split distribution (NUM-EFSD):" CU splits data-flows of all UEs equally among the selected paths, \textit{i.e.}, $\beta_k^{i,j}[t]=1/|\mathcal{P}_k|,\forall (i,j)\in\mathcal{P}_k$.
    \item ``NUM with the nearest RU selection (NUM-NRU):" Under Algorithm \ref{alg_JFCS}, each UE $k$ selects only the nearest RU for the data transmission, \textit{i.e.} $\beta_k^{i,j}[t]=1$ if RU $(i,j)$ is the nearest RU to UE $k$.
\end{itemize}

\subsection{Numerical Results of Algorithm \ref{alg_JFCS}}\label{sec_NumericalResults_B}
\begin{figure}[!ht]%
\centering
\subfigure[Impact of $\varphi$ on congestion control rate]{%
\label{fig:Convergence-a}%
\includegraphics[width=0.91\columnwidth]{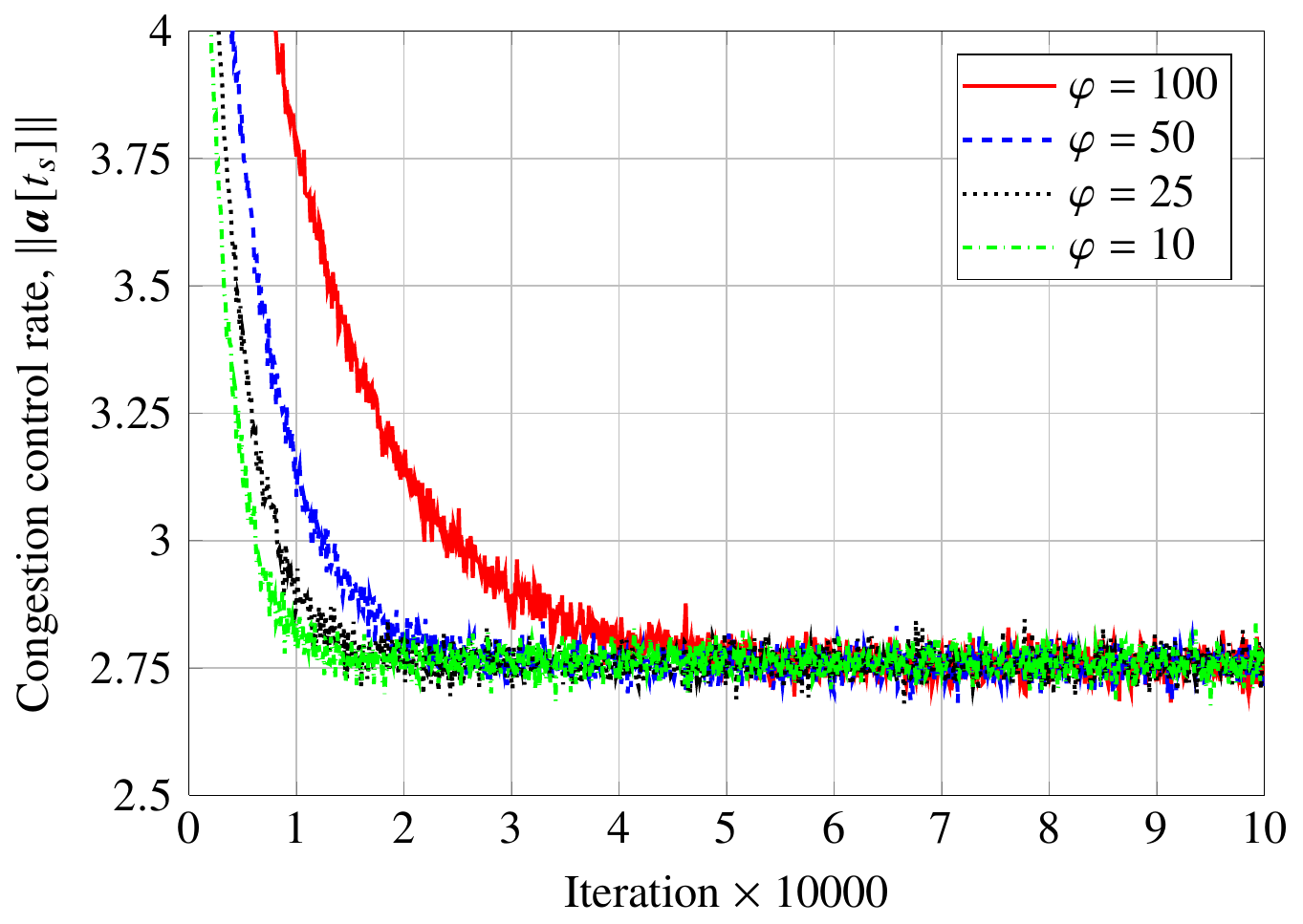}
}%
\hspace{4pt}%
\subfigure[Impact of $\lambda$ on estimated utility]{%
\label{fig:Convergence-b}%
\includegraphics[width=0.9\columnwidth]{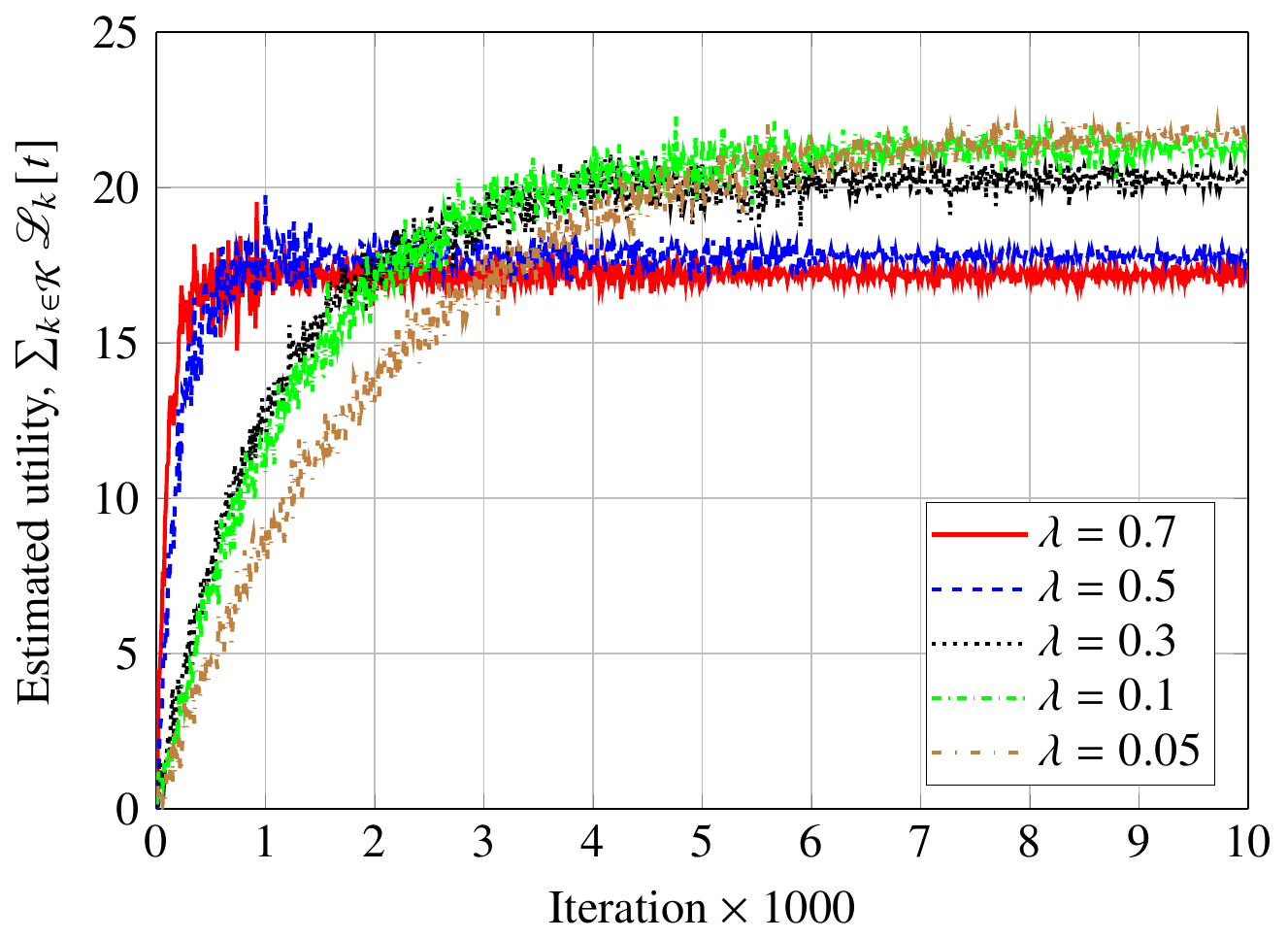}} 
\caption[]{\small Convergence behavior of Algorithm \ref{alg_JFCS} with ZFBF.}%
\label{fig:Convergence}%
\end{figure}

We first study the impacts of $\varphi$ and $\lambda$ on the convergence behavior of Algorithm \ref{alg_JFCS} in Fig. \ref{fig:Convergence}. From Fig. \ref{fig:Convergence}(a), it can be observed that the congestion control rates for  different values of the scaling factor $\varphi$ converge to the same optimal solution, and $\|\boldsymbol{a}[t_s]\|$ is almost independent of $\varphi$. In addition, increasing $\varphi$ results in a smaller divergence
of the steady-state congestion control rate (see Theorem \ref{Theo_2}), but also slows down the convergence rate of Algorithm \ref{alg_JFCS}. The reason is attributed to the fact that for a large $\varphi$, the network utility function $\sum_{k\in\mathcal{K}}U_k(a_k[t_s])$ in \eqref{eq:JFCS3a} will prevail over the Lyapunov drift function $\Delta L[t_s]$, which requires more iterations to guarantee network stability. In Fig. \ref{fig:Convergence}(b), we increase  the trade-off factor $\lambda$ (\textit{i.e.} Boltzmann
temperature) from 0.05 to 0.7. The result shows that the larger the value of $\lambda$, the better the estimated utility that can be achieved with the cost of lower convergence speed of the RL process. From \eqref{eq_bestresponse2}, the paths associated with the highest estimated regret $\hat{\theta}_k^{i,j}[t]$ will be selected to minimize the best response function $f(\hat{\boldsymbol{\theta}}[t])$. Conversely, a low value of $\lambda$ can speed up convergence by allocating traffic data uniformly to all paths but leads to a very sub-optimal solution.

\begin{figure}[t]
	\centering
	\includegraphics[width=1\columnwidth,trim={0cm 0.0cm 0cm 0.0cm}]{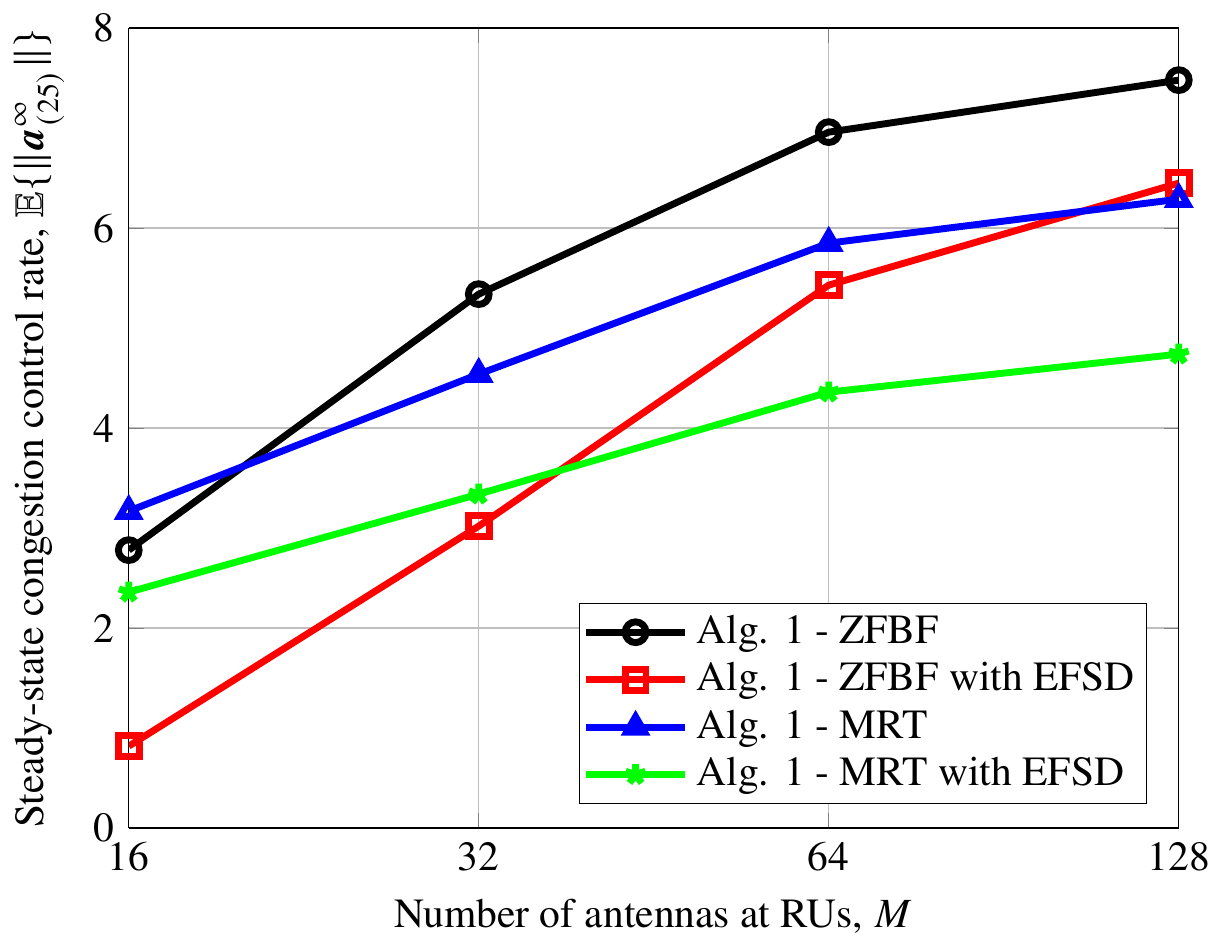}
	\caption{\small Performance of Algorithm \ref{alg_JFCS} with different transmission strategies versus the number of antennas at RUs, $M \equiv M_{i,j}, \forall (i,j)$.}
	\label{fig:ComparisionofAlg1}
\end{figure}

In Fig. \ref{fig:ComparisionofAlg1}, we evaluate the performance of Algorithm \ref{alg_JFCS} with different transmission strategies, \textit{namely} MRT and ZFBF. For a fixed $\varphi = 25$, we vary the number of antennas at RUs $M \equiv  M_{i,j}, \forall (i,j)$ from 16 to 128. For each transmission design, we also plot the steady-state congestion control rate $\mathbb{E}\{\|\boldsymbol{a}_{(25)}^{\infty}\|\}$ with the equal flow-split distribution. As seen from Fig. \ref{fig:ComparisionofAlg1} that the steady-state congestion control rate of all schemes increases as $M$ increases. Unsurprisingly, Algorithm \ref{alg_JFCS} with ZFBF offers better performance in terms of congestion control rate than that of MRT when the number of antennas at RUs is sufficiently large to cancel the inter-user interference transmitted by the same RU. It is obvious that the higher the effective data rate of a data-flow in the downlink, the lower the total  queue-length of that data-flow (or user), resulting in a higher congestion control rate.

Since Algorithm \ref{alg_JFCS} with MRT is based on the IA method that requires high computation complexity and relies on existing convex optimization solvers, we  provide only the performance of Algorithm \ref{alg_JFCS} with ZFBF in the following section.

\subsection{Performance Comparison}\label{sec_NumericalResults_C}
\begin{figure}[t]
	\centering
	\includegraphics[width=1\columnwidth,trim={0cm 0.0cm 0cm 0.0cm}]{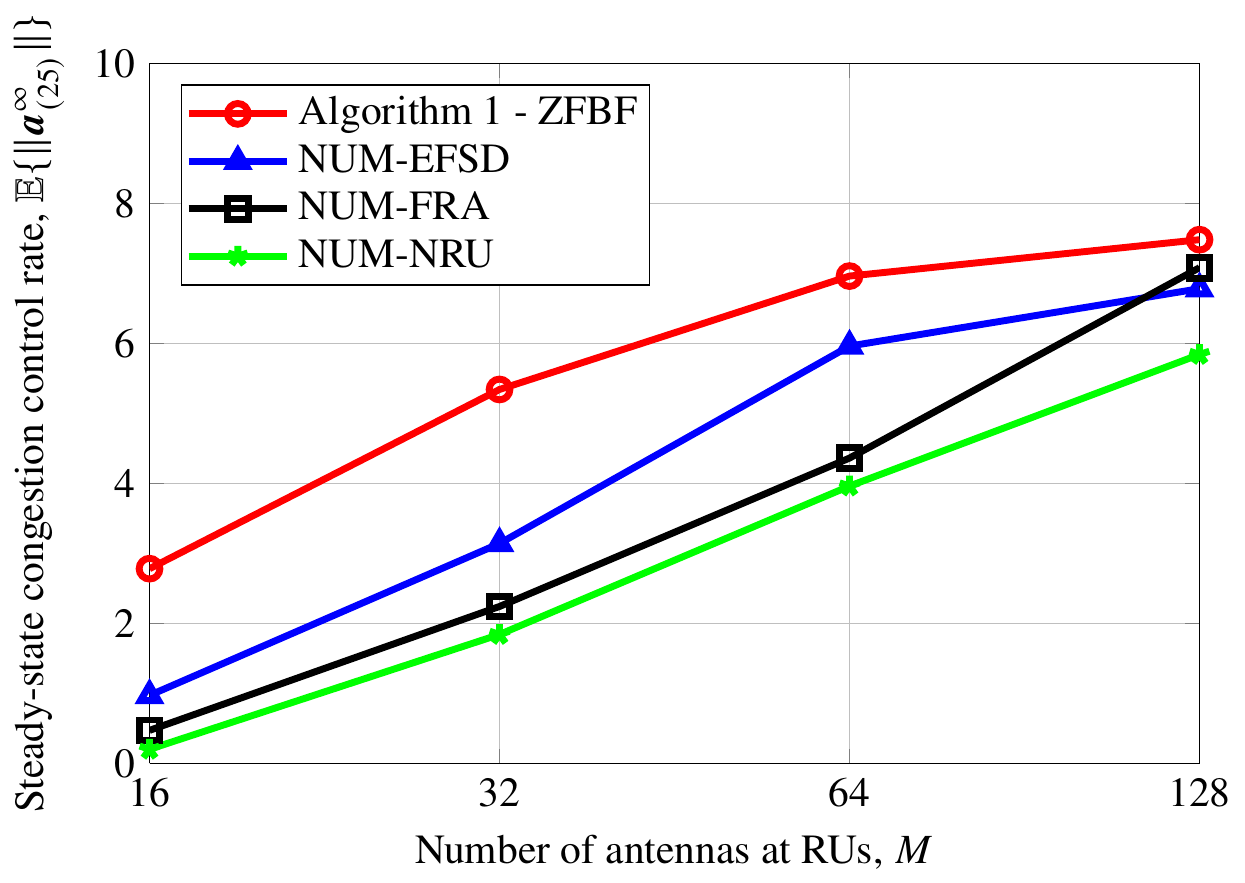}
	\caption{\small The steady-state congestion control rate with respect to the number of antennas at RUs, $M \equiv M_{i,j}, \forall (i,j)$.}
	\label{fig:CongestionratevsMout}
\end{figure}

Next, we show the performance comparison in terms of the steady-state congestion  control rate $\mathbb{E}\{\|\boldsymbol{a}_{(25)}^{\infty}\|\}$ among the considered schemes versus the number of antennas at RUs in Fig. \ref{fig:CongestionratevsMout}. We fix $\varphi = 25$ and vary $M$ from 16 to 128 to investigate the impact of the physical factor. As $M$ increases,  the downlink instantaneous achievable rates of all UEs also significantly increase since more degrees of freedom are added to leverage multi-user diversity, resulting in lower queue-lengths. For a fixed value of $\varphi$, the steady-state congestion  control rate vector  increases monotonically
with $M$. Clearly,  Algorithm \ref{alg_JFCS} outperforms the benchmark schemes in all ranges of $M$, and the gap is deeper when $M$ is small. In addition, the NUM-FRA and NUM-NRU, which fairly allocate the power budget and fix the path selection to UEs, respectively, provide the worst performance. These observations demonstrate the effectiveness of the proposed Algorithm \ref{alg_JFCS} by jointly optimizing the flow-split distribution, congestion control, scheduling and radio resource allocation.

Lastly, the impacts of scaling factor $\varphi$ on the steady-state total queue-length $\mathbb{E}\{\|\hat{\mathbf{q}}_{(\varphi)}^{\infty}\|_1\}$ and average worst-case delay (\textit{i.e.}, the delay of slowest data-flow) are plotted in Figs. \ref{fig:queue-length} and \ref{fig:latency}, respectively. It can be seen from Fig. \ref{fig:queue-length} that the steady-state total queue-length of all schemes monotonically scales as $\mathcal{O}(\varphi) + \mathcal{O}(\sqrt{\varphi})$, which confirms our theoretical results in Corollary \ref{corollary1}. We recall from Theorem \ref{Theo_2} that the utility-optimality gap can be narrowed by increasing $\varphi$, but with the cost of higher delay, as shown in Fig. \ref{fig:latency}. When $\varphi$ is larger than 25, all the considered schemes violate the  maximum allowable average delay of $\bar{d}=10$ ms. It implies that the data traffic cannot be completely transmitted to UEs in each time-frame. Nevertheless, Algorithm \ref{alg_JFCS} still provides  the best performance out of the schemes considered.

\begin{figure}[t]
	\centering
	\includegraphics[width=1\columnwidth,trim={0cm 0.0cm 0cm 0.0cm}]{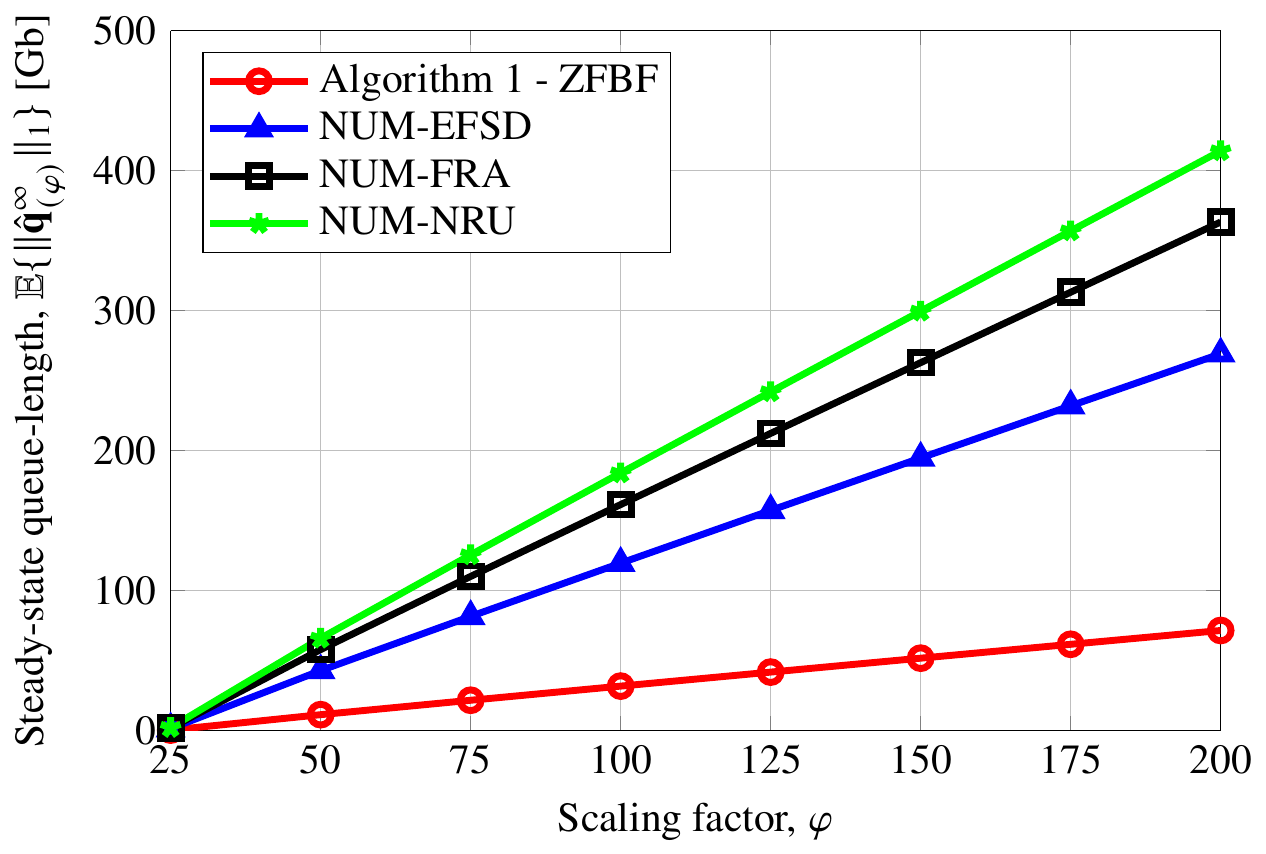}
	\caption{\small The steady-state total queue-length with respect to $\varphi$.}
	\label{fig:queue-length}
\end{figure}

\begin{figure}[t]
	\centering
	\includegraphics[width=1\columnwidth,trim={0cm 0.0cm 0cm 0.0cm}]{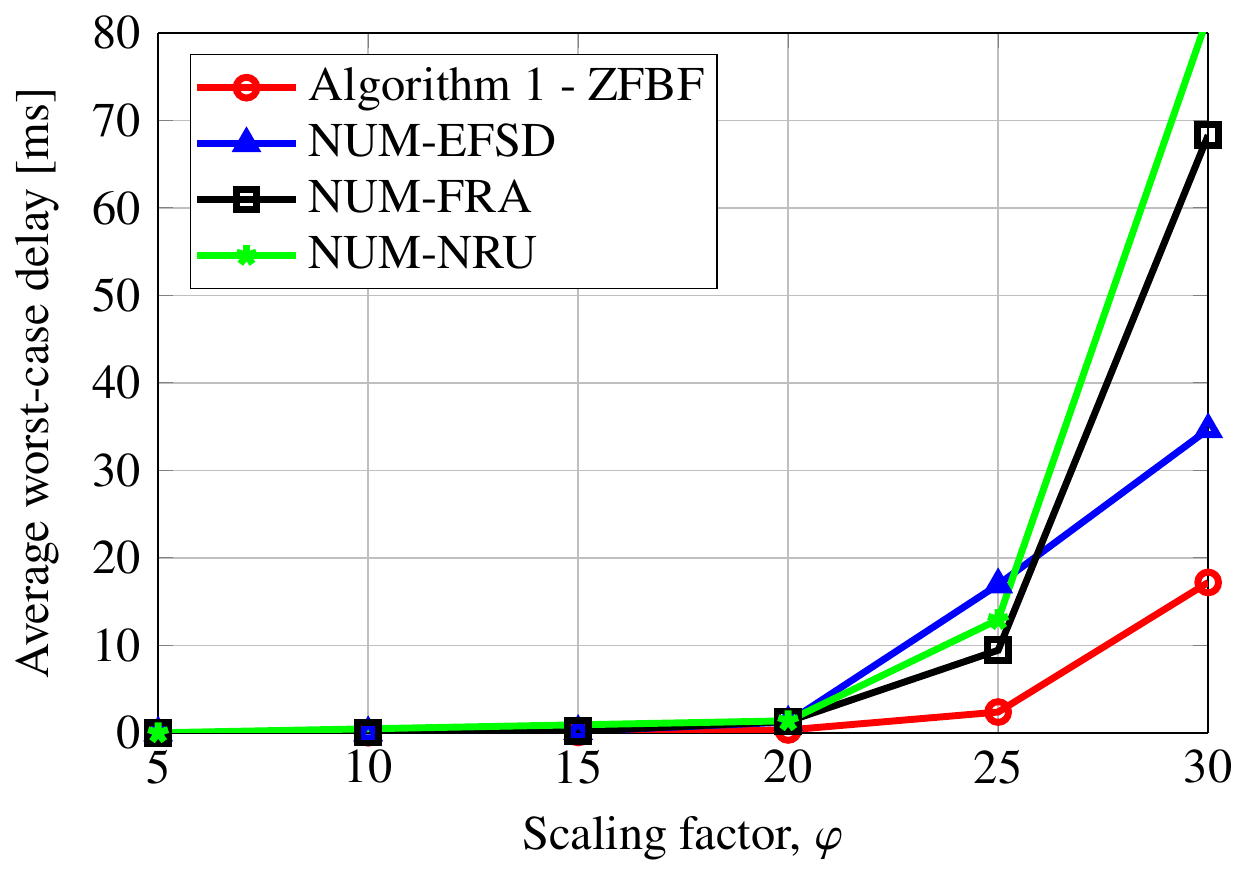}
	\caption{\small Average worst-case delay with respect to $\varphi$.}
	\label{fig:latency}
\end{figure}

\section{Conclusion}\label{sec_Conclusion}
We have proposed a new holistic multi-layer optimization framework, called JFCS, to enable intelligent traffic steering in a hierarchical O-RAN architecture. In particular, we have developed an intelligent resource management algorithm based on network utility maximization and stochastic optimization to  efficiently and adaptively direct traffic to appropriate RUs by jointly optimizing the flow-split
distribution, congestion control and scheduling. JFCS is proved to achieve fast convergence, long-term utility-optimality and significant delay reduction compared to state-of-the-art approaches. To that end, the insights in this work will foster future studies in this area, especially in the design of more advanced AI/ML solutions to achieve enhanced control and flexibility in O-RAN.

\appendices
\renewcommand{\thesectiondis}[2]{\Alph{section}:}
\section{Derivation of Inequality} \label{app:DerivationofInequ}
\renewcommand{\theequation}{\ref{app:DerivationofInequ}.\arabic{equation}}\setcounter{equation}{0}
We will find the concave lower bound of $r_{k}^{i,j}[t_s]$. By \cite[Appendix A]{Dinh:TCOMM:2017}, it is true that the function $r(x,y) = -\ln(1-x^2/y)$ is convex in the domain $y>x^2$ with $x,y\in\mathbb{R}_+$. The global concave lower bound of $r(x,y)$ at the feasible point ($\bar{x}, \bar{y}$) is given as
\begin{align}
    r(x,y) &\geq r(\bar{x},\bar{y}) + \Bigl<\Bigr(\frac{\partial r(\bar{x},\bar{y})}{\partial \bar{x}},\frac{\partial r(\bar{x},\bar{y})}{\partial \bar{y}}\Bigl), (x-\bar{x},y-\bar{y})  \Bigl>\nonumber\\
    & = r(\bar{x},\bar{y}) - \frac{\bar{x}^2}{\bar{y}-\bar{x}^2} + 2\frac{\bar{x}x}{\bar{y}-\bar{x}^2} - \frac{\bar{x}^2}{\bar{y}-\bar{x}^2}\frac{y}{\bar{y}}\label{eq_IAappro}
\end{align}
by applying the first-order Taylor approximation. By the fact that $\ln\bigl(1+\frac{x^2}{z}\bigr)=-\ln\bigl(1-\frac{x^2}{z+x^2}\bigr)$ and substituting $y=z+x^2$, $\bar{y}=\bar{z}+\bar{x}^2$,  $x=\sqrt{v}$ and $\bar{x}=\sqrt{\bar{v}}$ into \eqref{eq_IAappro}, we obtain
\begin{align}
    r(v,z) &\triangleq \ln\bigl(1+\frac{v}{z}\bigr) \geq r(\bar{v},\bar{z}) - \frac{\bar{v}}{\bar{z}} + 2\frac{\sqrt{\bar{v}}\sqrt{v}}{\bar{z}} - \frac{\bar{v}(z+v)}{\bar{z}(\bar{z}+\bar{v})}
    \nonumber\\
    &:=\bar{r}(v,z;\bar{v},\bar{z})\label{eq_IAapproConcave}
\end{align}
where $\bar{r}(v,z;\bar{v},\bar{z})$ is concave and $\bar{r}(\bar{v},\bar{z};\bar{v},\bar{z}) = r(\bar{v},\bar{z})$ whenever $(v,z)=(\bar{v},\bar{z})$.

\renewcommand{\thesectiondis}[2]{\Alph{section}:}
\section{Proof of Theorem \ref{Theo_1}} \label{App_B}
\renewcommand{\theequation}{\ref{App_B}.\arabic{equation}}\setcounter{equation}{0}
For a given $\varphi$, the  quadratic Lyapunov function defined in Section \ref{sec:ORANNUM_A} is rewritten with respect to $\hat{\mathbf{q}}_{(\varphi)}[t_s]$ as: 
 $L(\hat{\mathbf{q}}_{(\varphi)}[t_s]) = \frac{1}{2\tau^2}\|\hat{\mathbf{q}}_{(\varphi)}[t_s] - \hat{\mathbf{q}}^*_{(\varphi)}\|_2^2$. Following \cite[Theorem 3]{LiuJSAC2017}, the mean Lyapunov drift from  time-slot $t_s$ to $t_{s+1}$ is computed as
\begin{align}\label{eq_B1}
  &\Delta \bar{L}(\hat{\mathbf{q}}_{(\varphi)}[t_s]) \nonumber\\
  &= \mathbb{E}\{\Delta L(\hat{\mathbf{q}}_{(\varphi)}[t_s])\} 
  = \mathbb{E}\{L(\hat{\mathbf{q}}_{(\varphi)}[t_{s+1}]) - L(\hat{\mathbf{q}}_{(\varphi)}[t_s])\}  \nonumber\\
  &= \frac{1}{2\tau^2}\mathbb{E}\Bigl\{\bigl(\hat{\mathbf{q}}_{(\varphi)}[t_{s+1}] +\hat{\mathbf{q}}_{(\varphi)}[t_{s}] 
 - 2\hat{\mathbf{q}}^*_{(\varphi)} \bigr)^{\sfT}\nonumber\\
 &\quad \times\bigl(\hat{\mathbf{q}}_{(\varphi)}[t_{s+1}]- \hat{\mathbf{q}}_{(\varphi)}[t_{s}]\bigr)\Bigr\}\nonumber\\
  &\leq \frac{1}{2\tau}\mathbb{E}\Bigl\{\bigl(2\hat{\mathbf{q}}_{(\varphi)}[t_{s}] + \bigr(\bfa[t_s] - \bfr(\mathbf{w}[t_s])\bigl)\tau - 2\hat{\mathbf{q}}^*_{(\varphi)} \bigr)^{\sfT}\nonumber\\
  &\quad\times\bigl(\bfa[t_s] - \bfr(\mathbf{w}[t_s])\bigr)\Bigr\}\nonumber\\
  &=\underbrace{\frac{1}{2}\bbE\{\|\bfa[t_s] - \bfr(\mathbf{w}[t_s])\|_2^2\}}_{\triangleq \mathsf{B}_1} \nonumber\\
  &\quad + \underbrace{\frac{1}{\tau}\bbE\{(\hat{\mathbf{q}}_{(\varphi)}[t_{s}]-\hat{\mathbf{q}}^*_{(\varphi)})^\sfT\bigl(\bfa[t_s] - \bfr(\mathbf{w}[t_s])\bigr)\}}_{\triangleq \mathsf{B}_2}
\end{align}
by using the inequalities: $([x]^+)^2 \leq x^2$ and $x^2-y^2 = (x+y)(x-y)$, and the fact that $\hat{\mathbf{q}}_{(\varphi)}[t_{s+1}]- \hat{\mathbf{q}}^*_{(\varphi)}=\hat{\mathbf{q}}_{(\varphi)}[t_s] - \hat{\mathbf{q}}^*_{(\varphi)} + \bigr(\bfa[t_s] - \bfr(\mathbf{w}[t_s])\bigl)\tau.$

We first focus on providing the expected bound of $\mathsf{B}_1$ as
\begin{align}\label{eq_B2}
\mathsf{B}_1 &= \frac{1}{2}\bbE\{\|\bfa[t_s]\|_2^2 - 2\bfa[t_s]^{\sfT}\bfr(\mathbf{w}[t_s]) + \|\bfr(\mathbf{w}[t_s])\|_2^2\} \nonumber\\
&\leq \frac{1}{2}\bbE\{\|\bfa[t_s]\|_2^2  + \|\bfr(\mathbf{w}[t_s])\|_2^2\} \nonumber\\
&\leq \frac{K}{2}\bigl(A_1^{\max} + (r^{\max})^2\bigr) \triangleq \mathsf{B}_1^{\UB}
\end{align}
where the last inequality follows from Assumption \ref{assump_2}. To bound $\mathsf{B}_2$, we first rewrite it equivalently as
\begin{align}\label{eq_B3}
\mathsf{B}_2 =& \frac{1}{\tau}(\hat{\mathbf{q}}_{(\varphi)}[t_{s}]-\hat{\mathbf{q}}^*_{(\varphi)})^\sfT\bigl(\bbE\bigl\{\bfa[t_s]\} - \bfr^*\bigr) \nonumber\\
&+ \frac{1}{\tau}\bbE\bigl\{(\hat{\mathbf{q}}_{(\varphi)}[t_{s}]-\hat{\mathbf{q}}^*_{(\varphi)})^\sfT\bigl(\bfr^*
- \bfr(\mathbf{w}[t_s])\bigr)\bigr\}.
\end{align}
From \eqref{eq:JFCS2}, it follows that $(\hat{\mathbf{q}}_{(\varphi)}[t_{s}]-\hat{\mathbf{q}}^*_{(\varphi)})^\sfT\bigl(\bbE\bigl\{\bfa[t_s]\} - \bfr^*\bigr)\leq 0$.
By applying the Cauchy–Schwarz inequality, i.e. $|\mathbf{x}^\sfT\mathbf{y}|\leq \| \mathbf{x}\|_2\| \mathbf{y}\|_2$, to the first term in \eqref{eq_B3}, we have
\begin{align}
 &\frac{1}{\tau}(\hat{\mathbf{q}}_{(\varphi)}[t_{s}]-\hat{\mathbf{q}}^*_{(\varphi)})^\sfT\bigl(\bbE\bigl\{\bfa[t_s]\} - \bfr^*\bigr) \nonumber\\
 &\leq -\frac{1}{\tau}\sum_{k\in\mathcal{K}}|\hat{q}_{(\varphi),k}[t_{s}]-\hat{q}^*_{(\varphi),k}||a_k[t_s] - r^*_k|.
\end{align}
By Assumption \ref{assp:1} on $\Psi$-smooth and Step 4 of Algorithm \ref{alg_JFCS}, it is true that $a_k[t_s] - r^*_k=U_k^{'-1}\bigr(\frac{\hat{q}_{(\varphi),k}[t_{s}]}{\varphi\tau}\bigl)$ $-U_k^{'-1}\bigr(\frac{\hat{q}_{(\varphi),k}^*[t_{s}]}{\varphi\tau}\bigl)\leq 0$  and $\bigr|U_k^{'}\bigr(\frac{\hat{q}_{(\varphi),k}[t_{s}]}{\varphi\tau}\bigl)-U_k^{'}\bigr(\frac{\hat{q}_{(\varphi),k}^*[t_{s}]}{\varphi\tau}\bigl)\bigr| \leq \Psi\bigl|\frac{\hat{q}_{(\varphi),k}[t_{s}]}{\varphi\tau} - \frac{\hat{q}_{(\varphi),k}^*[t_{s}]}{\varphi\tau}\bigr|$. In addition, we have $\bigr|U_k^{'-1}\bigr(\frac{\hat{q}_{(\varphi),k}[t_{s}]}{\varphi\tau}\bigl)-U_k^{'-1}\bigr(\frac{\hat{q}_{(\varphi),k}^*[t_{s}]}{\varphi\tau}\bigl)\bigr| \geq \frac{1}{\Psi}\bigl|\frac{\hat{q}_{(\varphi),k}[t_{s}]}{\varphi\tau} - \frac{\hat{q}_{(\varphi),k}^*[t_{s}]}{\varphi\tau}\bigr|$ due to the inverse function lemma. From the fact that $(\hat{\mathbf{q}}^*_{(\varphi)})^\sfT\bfr^*
- (\hat{\mathbf{q}}^*_{(\varphi)})^\sfT\bfr(\mathbf{w}[t_s])\geq 0$, we can further bound $\mathsf{B}_2$ as
\begin{align}\label{eq_B5}
&\mathsf{B}_2 \leq -\frac{1}{\tau^2\Psi\varphi}\|\hat{\mathbf{q}}_{(\varphi)}[t_{s}]-\hat{\mathbf{q}}^*_{(\varphi)}\|_2^2 \nonumber\\
&\qquad + \frac{1}{\tau}\bbE\bigl\{(\hat{\mathbf{q}}_{(\varphi)}[t_{s}])^\sfT\bigl(\bfr^*
- \bfr(\mathbf{w}[t_s])\bigr)\bigr\}
\end{align}
where the term $\bbE\bigl\{(\bfr^*
- \bfr(\mathbf{w}[t_s]))\bigl\}$ is a constant with respect to $\hat{\mathbf{q}}_{(\varphi)}[t_{s}]$. Substituting \eqref{eq_B2} and \eqref{eq_B5} into \eqref{eq_B1} yields 
\begin{align}\label{eq_B6}
  \Delta \bar{L}(\hat{\mathbf{q}}_{(\varphi)}[t_s]) \leq & -\frac{1}{\tau^2\Psi\varphi}\|\hat{\mathbf{q}}_{(\varphi)}[t_{s}]-\hat{\mathbf{q}}^*_{(\varphi)}\|_2^2 + \mathsf{B}_1^{\UB} \nonumber\\
  & + \frac{1}{\tau}\bbE\bigl\{(\hat{\mathbf{q}}_{(\varphi)}[t_{s}])^\sfT\bigl(\bfr^*
- \bfr(\mathbf{w}[t_s])\bigr)\bigr\}.
\end{align}

We now compute the mean Lyapunov drift over $TT_f$ time-slots as
\begin{IEEEeqnarray}{rCl}\label{eq_B7}
  \Delta \bar{L}&&=\sum_{t=1}^T\sum_{s=1}^{T_f}\mathbb{E}\{L(\hat{\mathbf{q}}_{(\varphi)}[t_{s+1}]) - L(\hat{\mathbf{q}}_{(\varphi)}[t_s])|\hat{\mathbf{q}}_{(\varphi)}[1_1]\}\nonumber\\ 
  &&=\sum_{t=1}^T\sum_{s=1}^{T_f}\sum_{\hat{\mathbf{q}}_{(\varphi)}\geq 0}\Bigl(\Pro\bigl( \hat{\mathbf{q}}_{(\varphi)}[t_s]=\hat{\mathbf{q}}_{(\varphi)}|\hat{\mathbf{q}}_{(\varphi)}[1_1]\bigr)\nonumber\\
  &&\times\mathbb{E}\{L(\hat{\mathbf{q}}_{(\varphi)}[t_{s+1}]) - L(\hat{\mathbf{q}}_{(\varphi)}[t_s])|\hat{\mathbf{q}}_{(\varphi)}[t_s]=\hat{\mathbf{q}}_{(\varphi)}\} \Bigr).\qquad
\end{IEEEeqnarray}
Let us denote by $\rho_{\hat{\mathbf{q}}_{(\varphi)}}^{\infty}$ the stationary distribution of the Markov chain $\hat{\mathbf{q}}_{(\varphi)}[t_s]\geq 0$, i.e. $\rho_{\hat{\mathbf{q}}_{(\varphi)}}^{\infty}= \lim_{T\rightarrow\infty}\frac{1}{TT_f}\sum_{t=1}^T\sum_{s=1}^{T_f}\Pro\bigl( \hat{\mathbf{q}}_{(\varphi)}[t_s]=\hat{\mathbf{q}}_{(\varphi)}|\hat{\mathbf{q}}_{(\varphi)}[1_1]\bigr)$. By substituting \eqref{eq_B6} into \eqref{eq_B7} and dividing both side with $TT_f$, we have
\begin{IEEEeqnarray}{rCl}
&&\sum_{\hat{\mathbf{q}}_{(\varphi)}\geq 0}\rho_{\hat{\mathbf{q}}_{(\varphi)}}^{\infty}\Bigl(-\frac{1}{\tau^2\Psi\varphi}\|\hat{\mathbf{q}}_{(\varphi)}[t_{s}]-\hat{\mathbf{q}}^*_{(\varphi)}\|_2^2 + \mathsf{B}_1^{\UB} \nonumber\\
&&\quad +\, \frac{1}{\tau}(\hat{\mathbf{q}}_{(\varphi)}[t_{s}])^\sfT\bbE\bigl\{\bigl(\bfr^*
- \bfr(\mathbf{w}[t_s])\bigr)\bigr\}\Bigr) \nonumber\\
&&= -\frac{1}{\tau^2\Psi\varphi}\bbE\bigl\{\|\hat{\mathbf{q}}_{(\varphi)}^{\infty}-\hat{\mathbf{q}}^*_{(\varphi)}\|_2^2\bigl\} + \mathsf{B}_1^{\UB} \nonumber\\
&&\qquad + \frac{1}{\tau}\bbE\bigl\{(\hat{\mathbf{q}}_{(\varphi)}^{\infty})^\sfT\bigl(\bfr^*
- \bfr^{\infty}\bigr\} \geq 0
\end{IEEEeqnarray}
where $\bfr^{\infty} = \underset{r_k(\mathbf{w})\in \mathscr{C}_{\mathbf{H}[\infty]}, \forall k\in\mathcal{K}}{\argmax} \   \sum_{k\in\mathcal{K}} \hat{q}_k^{\infty}r_k(\mathbf{w})$. We note here that $(\hat{\mathbf{q}}_{(\varphi)}^{\infty})^\sfT \bfr^{\infty} = \underset{r_k(\mathbf{w})\in \mathscr{C}_{\mathbf{H}[\infty]}, \forall k\in\mathcal{K}}{\max}$ $\sum_{k\in\mathcal{K}} \hat{q}_k^{\infty}r_k(\mathbf{w}) \geq (\hat{\mathbf{q}}_{(\varphi)}^{\infty})^\sfT\bfr^*$, yielding
\begin{IEEEeqnarray}{rCl}
 \frac{1}{\tau^2\Psi\varphi}\bbE\bigl\{\|\hat{\mathbf{q}}_{(\varphi)}^{\infty}-\hat{\mathbf{q}}^*_{(\varphi)}\|_2^2\bigr\} - \mathsf{B}_1^{\UB} \leq 0
\end{IEEEeqnarray}
This implies that $\bbE\bigl\{\|\hat{\mathbf{q}}_{(\varphi)}^{\infty}-\hat{\mathbf{q}}^*_{(\varphi)}\|_2\bigr\} \leq \sqrt{\frac{K\tau^2\Psi}{2}\bigl(A_1^{\max} + (r^{\max})^2\bigr)}\sqrt{\varphi}$ where $\mathsf{B}_1^{\UB} = \frac{K}{2}\bigl(A_1^{\max} + (r^{\max})^2\bigr)$, showing the inequality \eqref{eq_theo1eq} in Theorem \ref{Theo_1}.

\renewcommand{\thesectiondis}[2]{\Alph{section}:}
\section{Proof of Theorem \ref{Theo_2}} \label{App_C}
\renewcommand{\theequation}{\ref{App_C}.\arabic{equation}}\setcounter{equation}{0}
To prove \eqref{eq_theo2a}, we first recall that $a_{(\varphi),k}^{\infty} - a^*_k=U_k^{'-1}\bigr(\frac{\hat{q}_{(\varphi),k}^{\infty}}{\varphi\tau}\bigl)$ $-U_k^{'-1}\bigr(\frac{\hat{q}_{(\varphi),k}^*}{\varphi\tau}\bigl)$ and $\bigr|U_k^{'}\bigr(\frac{\hat{q}_{(\varphi),k}^{\infty}}{\varphi\tau}\bigl)$ $-U_k^{'}\bigr(\frac{\hat{q}_{(\varphi),k}^*}{\varphi\tau}\bigl)\bigr| \geq \psi\bigl|\frac{\hat{q}_{(\varphi),k}^{\infty}}{\varphi\tau} - \frac{\hat{q}_{(\varphi),k}^*}{\varphi\tau}\bigr|$ using Assumption \ref{assp:1}. By the inverse function lemma, we have $\bigr|U_k^{'-1}\bigr(\frac{\hat{q}_{(\varphi),k}^{\infty}}{\varphi\tau}\bigl)$ $-U_k^{'-1}\bigr(\frac{\hat{q}_{(\varphi),k}^*}{\varphi\tau}\bigl)\bigr| \leq \frac{1}{\psi}\bigl|\frac{\hat{q}_{(\varphi),k}^{\infty}}{\varphi\tau} - \frac{\hat{q}_{(\varphi),k}^*}{\varphi\tau}\bigr|$, which yields
\begin{align}\label{eq_C1}
    \|\bfa^{\infty}_{(\varphi)} - \bfa^{*}\|_2 \leq \frac{1}{\psi\tau\varphi}\|\hat{\mathbf{q}}^{\infty}_{(\varphi)} - \hat{\mathbf{q}}^*_{(\varphi)}\|_2 \overset{\eqref{eq_theo1eq}}{\leq} \frac{\mathsf{C}_1}{\psi\tau} \frac{1}{\sqrt{\varphi}}.
\end{align}

Next, it is assumed that $U_k(\cdot)$ is twice continuously differentiable,  increasing, and strictly concave. If the utility function $ U(\bfa)$ has a maximizer $\bfa^*$, then
\begin{align}
    U(\bfa^*) - U(\bfa^{\infty}_{(\varphi)}) \leq  \frac{\Psi}{2}\|\bfa^* -\bfa^{\infty}_{(\varphi)}\|_2^2 \leq \frac{\Psi\mathsf{C}_1^2}{2\psi^2\tau^2} \frac{1}{\varphi}
\end{align}
where the last inequality follows from \eqref{eq_C1}. The proof is thus complete.

\begingroup
\balance
\bibliographystyle{IEEEtran}
\bibliography{IEEEfull}

\begin{thebibliography}{10}
\providecommand{\url}[1]{#1}
\csname url@samestyle\endcsname
\providecommand{\newblock}{\relax}
\providecommand{\bibinfo}[2]{#2}
\providecommand{\BIBentrySTDinterwordspacing}{\spaceskip=0pt\relax}
\providecommand{\BIBentryALTinterwordstretchfactor}{4}
\providecommand{\BIBentryALTinterwordspacing}{\spaceskip=\fontdimen2\font plus
\BIBentryALTinterwordstretchfactor\fontdimen3\font minus
  \fontdimen4\font\relax}
\providecommand{\BIBforeignlanguage}[2]{{%
\expandafter\ifx\csname l@#1\endcsname\relax
\typeout{** WARNING: IEEEtran.bst: No hyphenation pattern has been}%
\typeout{** loaded for the language `#1'. Using the pattern for}%
\typeout{** the default language instead.}%
\else
\language=\csname l@#1\endcsname
\fi
#2}}
\providecommand{\BIBdecl}{\relax}
\BIBdecl

\bibitem{SaadNet2020}
W.~Saad, M.~Bennis, and M.~Chen, ``A vision of {6G} wireless systems:
  Applications, trends, technologies, and open research problems,'' \emph{IEEE
  Network}, vol.~34, no.~3, pp. 134--142, 2020.

\bibitem{LetaiefComMag19}
K.~B. Letaief, W.~Chen, Y.~Shi, J.~Zhang, and Y.-J.~A. Zhang, ``The roadmap to
  {6G}: {AI} empowered wireless networks,'' \emph{IEEE Commun. Mag.}, vol.~57,
  no.~8, pp. 84--90, 2019.

\bibitem{ORAN2020}
{O-RAN Alliance}, ``{ORAN-WG1 O-RAN} architecture description v01.00.00.''
  \emph{Technical Specification}, Feb. 2020.

\bibitem{ORANASAMSUNG2019}
\BIBentryALTinterwordspacing
SAMSUNG, ``{ORAN - The Open Road to 5G},'' \emph{White Paper}, July 2019.
  [Online]. Available: \url{https://www.samsung.com/
  global/business/networks/insights/whitepapers /ORAN-the-open-road-to-5g/}
\BIBentrySTDinterwordspacing

\bibitem{ORANAlliance2019}
{O-RAN Alliance}, ``{ORAN Working Group 2}: {AI/ML} workflow description and
  requirements,'' \emph{Tech. Rep.}, Mar. 2019.

\bibitem{BonatiComMag21}
L.~Bonati, S.~D'Oro, M.~Polese, S.~Basagni, and T.~Melodia, ``Intelligence and
  learning in {O-RAN} for data-driven next{G} cellular networks,'' \emph{IEEE
  Commun. Mag.}, vol.~59, no.~10, pp. 21--27, 2021.

\bibitem{GavrilovskaWPC2020}
L.~Gavrilovska, V.~Rakovic, and D.~Denkovski, ``From cloud {RAN to ORAN},''
  \emph{Wireless Personal Communications}, Mar. 2020.

\bibitem{WangOpenRAN2019}
\BIBentryALTinterwordspacing
J.~Wang, H.~Roy, and C.~Kelly, ``{OpenRAN}: The next generation of radio access
  networks,'' \emph{Accesnture Startegy, Tech. Rep.}, Nov. 2019. [Online].
  Available: \url{https://telecominfraproject.com/openran/}
\BIBentrySTDinterwordspacing

\bibitem{KumarGC2020}
H.~Kumar, V.~Sapru, and S.~K. Jaisawal, ``{O-RAN} based proactive {ANR}
  optimization,'' in \emph{IEEE Global Commun. Conf. Workshops (IEEE GLOBECOM
  Wkshps.)}, 2020, pp. 1--4.

\bibitem{Pamukluicc2021}
T.~Pamuklu, M.~Erol-Kantarci, and C.~Ersoy, ``Reinforcement learning based
  dynamic function splitting in disaggregated green open {RANs},'' in
  \emph{IEEE Inter. Conf. Commun. (IEEE ICC 2021)}, 2021, pp. 1--6.

\bibitem{LeeGC2020}
H.~Lee, J.~Cha, D.~Kwon, M.~Jeong, and I.~Park, ``Hosting {AI/ML} workflows on
  {O-RAN RIC} platform,'' in \emph{IEEE Global Commun. Conf. Workshops (IEEE
  GLOBECOM Wkshps.)}, 2020, pp. 1--6.

\bibitem{RomeroINFOCOM2021}
J.~A. Ayala-Romero, A.~Garcia-Saavedra, X.~Costa-Perez, and G.~Iosifidis,
  ``Bayesian online learning for energy-aware resource orchestration in
  virtualized {RANs},'' in \emph{IEEE Conf. Comput. Commun. (IEEE INFOCOM)},
  2021, pp. 1--10.

\bibitem{YangTWC22}
Y.~Cao, S.-Y. Lien, Y.-C. Liang, K.-C. Chen, and X.~Shen, ``User access control
  in open radio access networks: A federated deep reinforcement learning
  approach,'' \emph{IEEE Trans. Wire. Commun.}, vol.~21, no.~6, pp. 3721--3736,
  2022.

\bibitem{MotallebTNSM23}
M.~Karbalaee~Motalleb, V.~Shah-Mansouri, S.~Parsaeefard, and O.~L.
  Alcaraz~López, ``Resource allocation in an {Open RAN} system using network
  slicing,'' \emph{IEEE Trans. Netw. and Ser. Manag.}, vol.~20, no.~1, pp.
  471--485, 2023.

\bibitem{LienTII23}
S.-Y. Lien and D.-J. Deng, ``Intelligent session management for {URLLC in 5G}
  open radio access network: A deep reinforcement learning approach,''
  \emph{IEEE Trans. Indus. Infor.}, vol.~19, no.~2, pp. 1844--1853, 2023.

\bibitem{NeelyToN2008}
M.~J. Neely, E.~Modiano, and C.-P. Li, ``Fairness and optimal stochastic
  control for heterogeneous networks,'' \emph{IEEE/ACM Trans. on Netw.},
  vol.~16, no.~2, pp. 396--409, 2008.

\bibitem{neely2010stochastic}
M.~J. Neely, ``Stochastic network optimization with application to
  communication and queueing systems,'' \emph{Synthesis Lectures Commun.
  Netw.}, vol.~3, no.~1, pp. 1--211, 2010.

\bibitem{EryilmazJSAC2006}
A.~Eryilmaz and R.~Srikant, ``Joint congestion control, routing, and {MAC} for
  stability and fairness in wireless networks,'' \emph{IEEE J. Sel. Areas in
  Commun.}, vol.~24, no.~8, pp. 1514--1524, 2006.

\bibitem{HabibiACCESS19}
M.~A. Habibi, M.~Nasimi, B.~Han, and H.~D. Schotten, ``A comprehensive survey
  of {RAN} architectures toward {5G} mobile communication system,'' \emph{IEEE
  Access}, vol.~7, pp. 70\,371--70\,421, 2019.

\bibitem{TangTWC15}
J.~Tang, W.~P. Tay, and T.~Q.~S. Quek, ``Cross-layer resource allocation with
  elastic service scaling in cloud radio access network,'' \emph{IEEE Trans.
  Wireless Commun.}, vol.~14, no.~9, pp. 5068--5081, 2015.

\bibitem{LuongTWC18}
P.~Luong, F.~Gagnon, C.~Despins, and L.-N. Tran, ``Joint virtual computing and
  radio resource allocation in limited fronthaul green {C-RANs},'' \emph{IEEE
  Trans. Wireless Commun.}, vol.~17, no.~4, pp. 2602--2617, 2018.

\bibitem{DouikTWC16}
A.~Douik, H.~Dahrouj, T.~Y. Al-Naffouri, and M.-S. Alouini, ``Coordinated
  scheduling and power control in cloud-radio access networks,'' \emph{IEEE
  Trans. Wireless Commun.}, vol.~15, no.~4, pp. 2523--2536, 2016.

\bibitem{DouikCL17}
A.~Douik, H.~Dahrouj, T.~Y. Al-Naffouri, and M.-S. Alouini, ``Low-complexity
  scheduling and power adaptation for coordinated cloud-radio access
  networks,'' \emph{IEEE Commun. Lett.}, vol.~21, no.~10, pp. 2298--2301, 2017.

\bibitem{AbiadTMC19}
M.~S. Al-Abiad, A.~Douik, and S.~Sorour, ``Rate aware network codes for cloud
  radio access networks,'' \emph{IEEE Trans. Mobile Comput.}, vol.~18, no.~8,
  pp. 1898--1910, 2019.

\bibitem{DouikTWC17}
A.~Douik, S.~Sorour, T.~Y. Al-Naffouri, and M.-S. Alouini, ``Rate aware
  instantly decodable network codes,'' \emph{IEEE Trans. Wireless Commun.},
  vol.~16, no.~2, pp. 998--1011, 2017.

\bibitem{AbiadTMC22}
M.~S. Al-Abiad, A.~Douik, S.~Sorour, and M.~J. Hossain, ``Throughput
  maximization in cloud-radio access networks using cross-layer network
  coding,'' \emph{IEEE Trans. Mobile Comput.}, vol.~21, no.~2, pp. 696--711,
  2022.

\bibitem{ORANWG22021}
\BIBentryALTinterwordspacing
O-RAN.WG2.Use-Case-Requirements-v02.01, ``{Non-RT RIC \& A1} interface: Use
  cases and requirements,'' \emph{Technical Specification}, Nov. 2021.
  [Online]. Available: \url{https://www.o-ran.org/specifications (accessed on
  10 November 2021)}
\BIBentrySTDinterwordspacing

\bibitem{AnwarIoT22}
M.~R. Anwar, S.~Wang, M.~F. Akram, S.~Raza, and S.~Mahmood, ``{5G}-enabled
  {MEC:} {A} distributed traffic steering for seamless service migration of
  internet of vehicles,'' \emph{IEEE Internet of Things J.}, vol.~9, no.~1, pp.
  648--661, 2022.

\bibitem{KavehmadavaniTWC23}
F.~Kavehmadavani, V.-D. Nguyen, T.~X. Vu, and S.~Chatzinotas, ``Intelligent
  traffic steering in beyond {5G Open RAN} based on {LSTM} traffic
  prediction,'' \emph{IEEE Trans. Wireless Commun.}, pp. 1--1, 2023.

\bibitem{VuTWC2019}
T.~K. Vu, M.~Bennis, M.~Debbah, and M.~Latva-Aho, ``Joint path selection and
  rate allocation framework for {5G} self-backhauled mm-wave networks,''
  \emph{IEEE Trans. Wireless Commun.}, vol.~18, no.~4, pp. 2431--2445, 2019.

\bibitem{SinghCOMML2016}
S.~Singh, M.~Geraseminko, S.-p. Yeh, N.~Himayat, and S.~Talwar, ``Proportional
  fair traffic splitting and aggregation in heterogeneous wireless networks,''
  \emph{IEEE Commun. Lett.}, vol.~20, no.~5, pp. 1010--1013, 2016.

\bibitem{Ngo:TWC:Mar2017}
H.~Q. {Ngo}, A.~{Ashikhmin}, H.~{Yang}, E.~G. {Larsson}, and T.~L. {Marzetta},
  ``Cell-free massive {MIMO} versus small cells,'' \emph{IEEE Trans. Wireless
  Commun.}, vol.~16, no.~3, pp. 1834--1850, Mar. 2017.

\bibitem{KeyInforcom2007}
P.~Key, L.~Massoulie, and D.~Towsley, ``Path selection and multipath congestion
  control,'' in \emph{IEEE Conf. Comput. Commun. (IEEE INFO- COM)}, 2007, pp.
  143--151.

\bibitem{LiuJSAC2017}
J.~Liu, A.~Eryilmaz, N.~B. Shroff, and E.~S. Bentley, ``Understanding the
  impacts of limited channel state information on massive {MIMO} cellular
  network optimization,'' \emph{IEEE J. Sel. Areas Commun.}, vol.~35, no.~8,
  pp. 1715--1727, 2017.

\bibitem{razaviyayn2014successive}
M.~Razaviyayn, ``Successive convex approximation: Analysis and applications,''
  Ph.D. dissertation, University of Minnesota, 2014.

\bibitem{BillingsleyProbability}
P.~Billingsley, \emph{Probability and Measure}, 3rd~ed.\hskip 1em plus 0.5em
  minus 0.4em\relax New York: Wiley, 1995.

\bibitem{TassiulasTAC92}
L.~Tassiulas and A.~Ephremides, ``Stability properties of constrained queueing
  systems and scheduling policies for maximum throughput in multihop radio
  networks,'' \emph{IEEE Trans. Automatic Control}, vol.~37, no.~12, pp.
  1936--1948, 1992.

\bibitem{BennisTWC2013}
M.~Bennis, S.~M. Perlaza, P.~Blasco, Z.~Han, and H.~V. Poor,
  ``Self-organization in small cell networks: A reinforcement learning
  approach,'' \emph{IEEE Trans. Wireless Commun.}, vol.~12, no.~7, pp.
  3202--3212, 2013.

\bibitem{leslie2003convergent}
D.~S. Leslie and E.~Collins, ``Convergent multiple-timescales reinforcement
  learning algorithms in normal form games,'' \emph{Anna. Appl. Prob.},
  vol.~13, no.~4, pp. 1231--1251, 2003.

\bibitem{Beck:JGO:10}
A.~Beck, A.~Ben-Tal, and L.~Tetruashvili, ``A sequential parametric convex
  approximation method with applications to nonconvex truss topology design
  problems,'' \emph{J. Global Optim.}, vol.~47, no.~1, pp. 29--51, May 2010.

\bibitem{Ben:2001}
A.~Ben-Tal and A.~Nemirovski, \emph{Lectures on Modern Convex
  Optimization.}\hskip 1em plus 0.5em minus 0.4em\relax Philadelphia: MPS-SIAM
  Series on Optimi., SIAM, 2001.

\bibitem{EryilmazToN2007}
A.~Eryilmaz and R.~Srikant, ``Fair resource allocation in wireless networks
  using queue-length-based scheduling and congestion control,'' \emph{IEEE/ACM
  Trans. Net.}, vol.~15, no.~6, pp. 1333--1344, 2007.

\bibitem{jafari2015study}
A.~H. Jafari, D.~L{\'o}pez-P{\'e}rez, M.~Ding, and J.~Zhang, ``Study on
  scheduling techniques for ultra dense small cell networks,'' in \emph{IEEE
  Veh. Technol. Conf. (VTC-Fall)}, 2015, pp. 1--6.

\bibitem{3GPP_Release_15}
3GPP, \emph{NR, Physical channels and modulation (Release 15)}, document 3GPP
  TS 38.211 version 15.2.0 Release 15, 2017.

\bibitem{SamarakoonTWC13}
S.~Samarakoon, M.~Bennis, W.~Saad, and M.~Latva-aho, ``Backhaul-aware
  interference management in the uplink of wireless small cell networks,''
  \emph{IEEE Trans. Wireless Commun.}, vol.~12, no.~11, pp. 5813--5825, 2013.

\bibitem{XiaojunJSAC06}
X.~Lin, N.~Shroff, and R.~Srikant, ``A tutorial on cross-layer optimization in
  wireless networks,'' \emph{IEEE J. Sel. Areas Commun.}, vol.~24, no.~8, pp.
  1452--1463, 2006.

\bibitem{StaiCOMML14}
E.~Stai and S.~Papavassiliou, ``User optimal throughput-delay trade-off in
  multihop networks under {NUM} framework,'' \emph{IEEE Commun. Lett.},
  vol.~18, no.~11, pp. 1999--2002, 2014.

\bibitem{Dinh:TCOMM:2017}
V.-D. Nguyen, T.~Q. Duong, H.~D. Tuan, O.-S. Shin, and H.~V. Poor, ``Spectral
  and energy efficiencies in full-duplex wireless information and power
  transfer,'' \emph{IEEE Trans. Commun.}, vol.~65, no.~5, pp. 2220--2233, May
  2017.

\end{thebibliography}
\endgroup

\end{document}